\documentclass{article}

\usepackage{aaai}

\usepackage[utf8]{inputenc} % allow utf-8 input
\usepackage{url}            % simple URL typesetting
\usepackage{booktabs}       % professional-quality tables
\usepackage{amsfonts}       % blackboard math symbols
\usepackage{nicefrac}       % compact symbols for 1/2, etc.
\usepackage{microtype}      % microtypography

\usepackage{graphicx}
\usepackage{subcaption}

\usepackage{amsmath,amsfonts,amssymb,amsthm}
\usepackage{mathtools}
\usepackage{amsmath}
\usepackage{amsthm}
\usepackage{amssymb}
\usepackage{mathrsfs}
\mathtoolsset{showonlyrefs}

\usepackage{algorithm}
\usepackage{algorithmic}

\usepackage[font=small,labelfont=bf]{caption}

\newtheorem{thm}{Theorem}[]
\newtheorem*{thm*}{Theorem}
\newtheorem{lemma}{Lemma}[]
\newtheorem*{lemma*}{Lemma}

\newcommand{\citet}[1]{\citeauthor{#1}~\shortcite{#1}}
\newcommand{\citep}{\cite}
\newcommand{\citealp}[1]{\citeauthor{#1}~\citeyear{#1}}
\usepackage[utf8]{inputenc}
 
\newcommand{\algorithmfootnote}[2][\footnotesize]{%
  \let\old@algocf@finish\@algocf@finish% Store algorithm finish macro
  \def\@algocf@finish{\old@algocf@finish% Update finish macro to insert "footnote"
    \leavevmode\rlap{\begin{minipage}{\linewidth}
    #1#2
    \end{minipage}}%
  }%
}
\title{Natural Option Critic}

\author{
  Saket Tiwari \\
  College of Information and Computer Sciences \\
  University of Massachusetts Amherst \\
  Amherst, MA 01003 \\
  \texttt{sakettiwari@umass.edu} \\
  \And
  Philip S. Thomas \\
  College of Information and Computer Sciences \\
  University of Massachusetts Amherst \\
  Amherst, MA 01003 \\
  \texttt{pthomas@cs.umass.edu}
}

\begin{document}

\maketitle

\begin{abstract}
The recently proposed \textit{option-critic} architecture \citep{DBLP:journals/corr/BaconHP16} provides a stochastic policy gradient approach to hierarchical reinforcement learning. 
Specifically, it provides a way to estimate the gradient of the expected discounted return with respect to parameters that define a finite number of temporally extended actions, called \textit{options}. 
In this paper we show how the option-critic architecture can be extended to estimate the \textit{natural} gradient \citep{Amari:1998:NGW:287476.287477} of the expected discounted return. 
To this end, the central questions that we consider in this paper are: {\bf 1)} what is the definition of the natural gradient in this context, {\bf 2)} what is the Fisher information matrix associated with an option's parameterized policy, {\bf 3)} what is the Fisher information matrix associated with an option's parameterized termination function, and {\bf 4)} how can a \textit{compatible function approximation} approach be leveraged to obtain natural gradient estimates for both the parameterized policy and parameterized termination functions of an option with per-time-step time and space complexity linear in the total number of parameters. Based on answers to these questions we introduce the natural option critic algorithm. Experimental results showcase improvement over the \textit{vanilla gradient} approach.
\end{abstract}

\section{Introduction}
Hierarchical reinforcement learning methods enable agents to tackle challenging problems by identifying reusable \textit{skills}---temporally extended actions---that simplify the task. 
For example, a robot agent that tries to learn to play chess by reasoning solely at the level of how much current to give to its actuators every 20ms will struggle to correlate obtained rewards with their true underlying cause. 
However, if this same agent first learns skills to move its arm, grasp a chess piece, and move a chess piece, then the task of learning to play chess (leveraging these skills) becomes tractable. 
Several mathematical frameworks for hierarchical reinforcement learning have been proposed, including \textit{hierarchies of machines} \citep{parr1998reinforcement}, MAXQ \citep{dietterich2000hierarchical}, and the options framework \citep{Sutton1999BetweenMA}. 
However, none of these frameworks provides a practical mechanism for \textit{skill discovery}: determining what skills will be useful for an agent to learn.
Although skill discovery methods have been proposed, they tend to be \textit{heuristic} in that they find skills that have a property that intuitively might make for good skills for some problems, but which do not follow directly from the primary objective of optimizing the expected discounted return (\citealp{pmlr-v70-machado17a}; \citealp{Simsek2008SkillCB}; \citealp{NIPS1994_887}; \citealp{NIPS2009_3683}).

The \textit{option-critic} architecture \citep{DBLP:journals/corr/BaconHP16}, stands out from other attempts at developing a general framework for skill discovery in that it searches for the skills that directly optimize the expected discounted return. 
Specifically, the option critic uses the aforementioned options framework, wherein a skill is called an \textit{option}, and it proposes parameterizing all aspects of the option and then performing stochastic gradient descent on the expected discounted return with respect to these parameters. 
The key insight that enables the option-critic architecture is a set of theorems that give expressions for the gradient of the expected discounted return with respect to the different parameters of an option.

One limitation of the option critic is that it uses ordinary (stochastic) gradient descent. 
In this paper we show how the option critic can be extended to use \textit{natural gradient descent} \citep{Amari:1998:NGW:287476.287477}, which exploits the underlying structure of the option-parameter space to produce a more informed update direction. 
The primary contributions of this work are theoretical: we define the natural gradients associated with the option critic, derive the \textit{Fisher information matrices} associated with an option's parameterized policy and termination function, and show how the natural gradients can be estimated with per-time-step time and space complexity linear in the total number of parameters. This is achieved by means of \textit{compatible function approximations}. We also analyze the performance of natural gradient descent based approach on various learning tasks.

\section{Preliminaries and Notation}
A \textit{reinforcement learning} (RL) agent interacts with an environment, modeled as a \textit{Markov decision process} (MDP), over a sequence of time steps $t \in \mathbb N_{\geq 0}$. A finite MDP is a tuple $(\mathcal S, \mathcal A, P, R, d_0, \gamma)$. $\mathcal S$ is the finite set of possible states of the environment. $S_t$ is the state of the environment at time $t$. $\mathcal A$ is the finite set of possible actions the agent can take. $A_t$ is the action taken by the agent at time $t$. $P:\mathcal{S \times A \times S} \to [0,1]$ is the transition function: $P(s,a, s^{\prime}) = \Pr(S_{t+1}{=}s^{\prime}| S_t {=} s, A_t {=} a)$, for all $t$. Meaning, $P(s, a, s')$ the probability of transitioning to state $s'$ given the agent takes action $a$ in state $s$.  $R_t$ denotes the reward at time $t$.  $R$ is the \textit{reward function}, $R: \mathcal{S \times A} \to \mathbb R $, where $R(s,a) = \mathbb{E}[R_t| S_t{=}a, A_t{=}a]$, i.e., the expected reward the agent receives given it took action $a$ in state $s$. We say that a process has ended when the environment enters a \textit{terminal state}, meaning for a terminal state $s$, $P(s, a, s^{\prime}) = 0$ and $R(s,a) = 0$ for all $s^{\prime} \in \mathcal{S} \setminus \{s\}$ and $a \in \mathcal{A}$. The process ends after $T$ steps and we call $T$ the \textit{horizon}. We say the process is \textit{infinite horizon} when there does not exist a finite $T$. $d_0$ is the initial state distribution, i.e., $d_0(s)=\Pr(S_0{=}s)$. The parameter $\gamma  \in [0,1]$ scales how the rewards are discounted over time. When a terminal state is reached, time is reset to $t=0$ and consequently a new initial state is sampled using $d_0$.

A policy, $\pi: \mathcal{S \times A} \to [0,1]$, represents the agent's decision making system: $\pi(s,a) = \Pr(A_t{=}a| S_t{=}s)$. Given a policy, $\pi$, and an MDP, $(\mathcal S, \mathcal A, P, R, d_0, \gamma)$, an episode, $H$ is a sequence of states of the environment, actions taken by the agent, and the rewards observed from the initial state, $S_0$, to the terminal state, $S_T$, i.e., $H=(S_0, A_0, R_0, S_1, A_1, R_1, ..., S_T, A_T, R_T)$. We also define the path that an agent takes to be a sequence of states and actions, i.e., a history without rewards, $X = (S_0, A_0, S_1, A_1, ..., S_T, A_T)$. Path $X$ is a random variable from the set of all possible paths, $\mathcal{X}$. The return of an episode $H$ is the discounted sum of all rewards, $g(H) = \sum_{t=0}^T \gamma^t R_t$.  We call $v_{\pi}$ the value function for the policy $\pi$, $v_{\pi}: \mathcal{S} \to \mathbb R$, where $v_{\pi}(s) = \mathbb E[\sum_{t{=}0}^{T} \gamma^t R_t |, S_0{=}s, \pi]$. We call $q_{\pi}$ the action-value function associated with policy $\pi$, $q_{\pi}: \mathcal{S} \times \mathcal{A} \to \mathbb R$, where $q_{\pi}(s, a) = \mathbb E[\sum_{t{=}0}^{T} \gamma^t R_t | S_0{=}s, A_0{=}a, \pi]$.
\subsection{Policy Gradient Framework}
The \textit{policy gradient framework} (\citealp{Sutton:1999:PGM:3009657.3009806}; \citealp{kondat:ac}) assumes the policy $\pi$, parametrized by $\theta$, is differentiable. The objective function, $\rho$, is defined with respect to a start state $s_0$, $\rho(\theta) = \mathbb E[\sum_{t = 0}^{T} \gamma^t R_t | d_0, \theta]$. The agent learns by updating the parameters $\theta$ approximately proportional to the gradient $\partial \rho/\partial \theta$, i.e., $\theta \leftarrow \alpha \partial \rho/\partial \theta$ where $\alpha$ is the \textit{learning rate} (LR): a scalar hyper-parameter.
\subsection{Option Critic framework}
The \textit{options framework} \citep{Sutton1999BetweenMA} formalizes the notion of temporal abstractions by introducing options. An option, $o$, from a set of options, $\mathcal O$, is a generalization of primitive actions. The intra-option policy $\pi_{o}: \mathcal{S \times A} \to [0,1]$ represents the agent's decision making while executing an option $o$: $\pi_{o}(s,a) = \Pr(A_t{=}a| S_t{=}s, O_t{=} o)$. 
Like primitive actions the agent executes an option at a state $S_t$ and the option terminates at another $S_{t + \tau}$, where $\tau$ is the duration for which the agent is executing the option: $o_{t}$. 
While in the option $o$, from state $S_t$ to $S_{t + \tau}$, the agent follows the policy $\pi_{o}$. 
Option $o$ terminates stochastically in state $s$ according to a distribution $\beta$. 
The framework puts restrictions on where an option can be initiated by defining an initiation state set, $\mathcal I_{o}$, for option $o$. The option $o$ is initiated in state $ s \in \mathcal I_{o}$ based on $\pi_{\mathcal O}(s)$, which is a policy over options defined as $\pi_{\mathcal O}: \mathcal S \times \mathcal O \to [0,1]$.  An initiation state set $\mathcal I_{o}$, an intra-option policy $\pi_{o}$ and a termination function $\beta_{o}: \mathcal{S} \to [0,1]$ comprise an option $o$. It is commonly assumed that all options are available everywhere and thereby we dispense with the notion of an initiation set.

The \textit{option critic framework} makes all the options available everywhere, and introduces policy-gradient theorems within the options framework. The option active at time step $t$ is $O_t$. The intra-option policies ($\pi_{o}$) and termination functions ($\beta_{o}$) are represented using differentiable functions parametrized by $\theta$ and $\vartheta$, respectively. The goal is to optimize the expected discounted return starting at state $s_0$ and option $o_0$. We re-define the objective function, $\rho$, for the option critic setting: $\rho(\mathcal O, \theta, \vartheta, s, o) = \mathbb E [\sum_{t=0}^{\infty} \gamma^t R_t | \mathcal O, \theta, \vartheta, S_0= s, O_0=o]$.

 Equations similar to those in the policy gradient framework \citep{Sutton:1999:PGM:3009657.3009806} are manipulated to derive gradients of the objective with respect to $\theta$ and $\vartheta$ in the option-critic framework. 
 The analogous state value function is $v_{\pi_{\mathcal O}}: \mathcal{S} \to \mathbb{R}$, where $v_{\pi_{\mathcal O}}(s) = \mathbb E[\sum_t \gamma^t R_t | S_0{=}s]$. 
 $v_{\pi_{\mathcal O}}(s)$ is the value of a state $s$, within the options framework, with the option set $\mathcal O$ and the policy over options $\pi_{\mathcal O}$. 
 The option-value function is $q_{\pi_{\mathcal O}}: \mathcal{S} \times \mathcal O \to \mathbb R$, where $q_{\pi_{\mathcal O}}(s, o) = \mathbb E[\sum_t \gamma^t R_t | S_0{=}s, O_0{=}o]$. Here, $q_{\pi_{\mathcal O}}(s, o)$ is the value of state $s$ when option $o$ is active with the option set $\mathcal O$. The state-option-action value function is $q_U: \mathcal{S} \times \mathcal O \times \mathcal{A} \to \mathbb{R}$, where $q_{U}(s, o, a) = \mathbb E[\sum_t \gamma^t R_t | S_0{=}s, O_0{=}o, A_0{=}a]$. Here, $q_{U}(s, o, a)$ is the value of executing action $a$ in the context of state-option pair $(s, o)$. The option-value function \textit{upon arrival} is $u:\mathcal O \times \mathcal{S} \to \mathbb{R}$, where $u(o, s^{\prime}) = \mathbb E[\sum_t \gamma^t R_t|S_1{=}s^{\prime}, O_0{=} o]$. Here, $u(o, s^{\prime})$ is the value of option $o$ being active upon the agent entering state $s'$. \citet{DBLP:journals/corr/BaconHP16} observe a consequence of the definitions:
\begin{equation}
    \label{eq:onarrival}
    u(o, s^{\prime}) = (1 - \beta_{o}(s^{\prime}))q_{\pi_{\mathcal O}}(s^{\prime}, o) + \beta_{o}(s^{\prime})v_{\pi_{\mathcal O}}(s^{\prime}).
\end{equation}
The main results presented by \citet{DBLP:journals/corr/BaconHP16} are the \textit{intra-option policy gradient theorem} and the \textit{termination gradient theorem}. The gradient of the expected discounted return with respect to $\theta$ and initial condition $(s_0, o_0)$ is:
\begin{equation}
    \frac{\partial q_{\pi_{\mathcal O}}(s_0, o_0)}{\partial \theta} = \sum_{s, o} \mu_{\mathcal O}(s, o) \sum_{a} \frac{\partial \pi_{o}(s, a, \theta)}{\partial \theta} q_U (s, o, a),
\end{equation}
where $\mu_{\mathcal O}(s, o)$ is the discounted weighting of state-option pair $(s, o)$ along trajectories starting from $(s_0, o_0)$ defined by $:\mu_{\mathcal O}(s, o) = \sum_{t = 0}^{\infty} \gamma^t \Pr(S_t{=}s, O_t{=}o | s_0, o_0)$. The gradient of the expected discounted return with respect to $\vartheta$ and initial condition $(s_1, o_0)$ is:
\begin{equation}
    \frac{\partial u(o_0, s_1)}{\partial \vartheta} = - \sum_{o, s^{\prime}} \mu_{\mathcal O}(s^{\prime}, o) \frac{\partial \beta_{o}(s^{\prime}, \vartheta) }{\partial \vartheta } a_{\mathcal O}(s^{\prime}, o),
\end{equation}
where $a_{\mathcal O}: \mathcal{S} \times \mathcal O \to \mathbb R$ is the advantage function over options such that $a_{\mathcal O}(s^{\prime}, o) = q_{\pi_{\mathcal O}}(s^{\prime}, o) - v_{\pi_{\mathcal O}}(s^{\prime})$. Here, $\mu_{\mathcal O}(s^{\prime}, o)$ is the discounted weighting of state option pair $(s', o)$ from $(s_1, o_0)$, i.e., according to a Markov chain shifted by one time step, defined by$:\mu_{\mathcal O}(s^{\prime}, o) = \sum_{t = 0}^{\infty} \gamma^t \Pr(S_{t+1}{=}s', O_t{=}o | s_1, o_0)$. The agent learns by updating parameters $\theta$ and $\vartheta$ in the direction approximately proportional to $\partial q_{\mathcal O}(s_0, o_0)/\partial \theta$ and $\partial u(o_0, s_1)/\partial \vartheta$, respectively. Meaning, it learns by updating $\theta \leftarrow \alpha_{\theta} \partial q_{\pi_{\mathcal O}}(s_0, o_0)/\partial \theta$ and $\vartheta \leftarrow \alpha_{\vartheta}\partial u(o_0, s_1)/\partial \vartheta $, where $\alpha_{\theta}$ and $\alpha_{\vartheta}$ are the learning rates for $\theta$ and $\vartheta$, respectively.
\subsection{Natural Actor Critic}
Natural gradient descent \citep{Amari:1998:NGW:287476.287477} exploits the underlying structure of the parameter space when defining the direction of steepest descent. It does so by defining the inner product $\langle \textbf{x, y} \rangle_{\theta}$ in the parameter space as: 
\begin{equation}
    \langle \textbf{x, y} \rangle_{\theta} = \textbf{x}^T G_{\theta} \textbf{y}, \label{eq:rspacedist}
\end{equation}
where $G_{\theta}$ is called the \textit{metric tensor}. 
Although the choice of $G_{\theta}$ remains open under certain conditions \citep{pmlr-v48-thomasb16} we choose the Fisher information matrix, as is common practice. The fisher information matrix distribution over random variable $X$, parametrized by policy parameters $\theta$, that lie on a Reimannian manifold
%link theta with the manifold M
 (\citealp{raofim}; \citealp{amarifim}):
\begin{equation}
    (G_{\theta})_{i, j} = \mathbb E\left[ \frac{\partial \ln \Pr(X; \theta)}{\partial \theta_i} \frac{\partial \ln \Pr(X; \theta)}{\partial \theta_j}\right] \label{eq:fim_def},
\end{equation}
where the expectation is over the distribution $\Pr(X)$ and $(G_{\theta})_{i,j}$ represents a matrix with its $i,j^{th}$ element being the expression as defined on the right hand side --- we use this notation to represent a matrix throughout the paper. \citet{Kakade:2001} makes the assumption that every policy, $\pi$, is ergodic and irreducible, therefore it has a well-defined \textit{stationary distribution} for each state $s$. Under this assumption, \citet{Kakade:2001} introduces the use of natural gradient for optimizing the expected reward over the parameters $\theta$ of policy $\pi$, as defined by $\rho(\theta) = \sum_{s, a} d^{\pi}(s) \pi(s, a, \theta) R(s, a)$. The \textit{natural gradient} for the objective function, $\rho$, is defined as:
\begin{equation} \widetilde{\nabla} \rho(\theta)  = G_{\theta}^{-1} \frac{\partial  \rho(\theta) }{\partial \theta} \label{eq:natgrad_kakade}.\end{equation}
The derivation of a closed form expression for $G_{\theta}$ for the parameter space of policy $\pi$, parametrized by $\theta$, is non-trivial as demonstrated for the limiting matrix of the infinite horizon problem in reinforcement learning \citep{bagnell2003covariant}. For a weight vector $w$ let $\hat{q}_w$ be an approximation of the state action value function $q(s,a)$, which has the form:
\[\hat{q}_w (s,a) = w^T \frac{\partial \ln \pi(s, a, \theta)}{\partial \theta}.\]
The mean squared error $\epsilon(w, \theta)$, for a weight vector $w$ and a given policy parametrized by $\theta$, is defined as:
\[\epsilon(w, \theta) = \sum_{s,a} d^{\pi}(s) \pi(s, a, \theta) (\hat{q}_w (s,a) - q_{\pi}(s,a))^2,\]
where $d^{\pi}(s) = \sum_{t =0}^{\infty} \gamma^t \Pr(S_t{=}s | \pi)$ is the discounted weighting of state $s$ in the infinite horizon problem. The weights $d^{\pi}(s)$ normalize to the stationary distribution for state $s$ under policy $\pi$ in the undiscounted setting where the MDP terminates at every time step $t$ with probability $1 - \gamma$. Theorem 1 as introduced by \citet{Kakade:2001} states that $\tilde{w}$ which minimizes the mean squared error, $\epsilon(w, \theta)$, is equal to the natural gradient as defined in \eqref{eq:natgrad_kakade}.

\citet{Kakade:2001} also demonstrates how natural policy gradient performs under the re-scaling of parameters. In addition to that, \citet{Kakade:2001} demonstrates how the natural gradient weights the components of $\widetilde{\nabla} \rho(\theta)$ uniformly, instead of using $d^{\pi}(s)$. We also point out that the natural gradient is independent to local re-parametrization of the model \citep{Pascanu2013RevisitingNG} and can be used in \textit{online learning} \citep{Degris_model-freereinforcement}.  
Natural gradients for reinforcement learning (\citealp{Peters2008,Bhatnagar2009}; \citealp{Degris_model-freereinforcement}), as well as more recent work in deep neural networks
(\citealp{desjardins2015natural}; \citealp{Pascanu2013RevisitingNG}; \citealp{thomas2017decoupling}; \citealp{Sun2017RelativeFI}) have shown to be effective in learning.

The Option-Critic architecture uses vanilla gradient to learn temporal abstraction and internal policies, which can be less data efficient compared to the natural gradient \citep{Amari:1998:NGW:287476.287477}. The natural gradient also overcomes the difficulty posed by the \textit{plateau phenomena} \citep{amariinfogeo}. We derive the metric tensors for the parameters in the option-critic architecture. 
Computing the complete Fisher information matrix or is expensive. 
We use a block-diagonal estimate of the Fisher information matrix as has been applied in the past to reinforcement learning \citep{NIPS2011_4449} and to neural networks (\citealp{NIPS2007_3234}; \citealp{Kurita1992IterativeWL}; \citealp{Martens2010DLV31043223104416}; \citealp{Pascanu2013RevisitingNG}; \citealp{Martens2015OptimizingNN}). Specifically, we estimate $G_{\theta}$ and $G_{\vartheta}$ separately, where $\theta$ and $\vartheta$ are the parameters of of the intra-option policy and the option termination function. These are then combined into a $(|\theta| + |\vartheta|) \times (|\theta| + |\vartheta|)$ sized estimate of the complete Fisher information matrix of the parameter space, where $|\theta|, |\vartheta|$ represent the size of vectors.

We also provide theoretical justification for the resulting algorithm inspired from the incremental natural actor critic algorithm \citep{Bhatnagar:2007:INA:2981562.2981576} (INAC) and its extension to include eligibility traces (\citealp{morimuranatgrad}; \citealp{Thomas2014BiasIN}).
\section{Start State Fisher Information Matrix Over Intra-Option Path Manifold}
We define path $X$ in the options framework for the infinite horizon problem as the sequence  of state-option-action tuples: $X = (S_0, O_0, A_0, S_1, O_1, A_1, ...)$. We use $\mathcal{X}$ to denote the set of all paths. We introduce the function $g:\mathcal X \to \mathbb R$ called the \textit{expected return over path}, where $g(x) = \mathbb E[\sum_{t=0}^T \gamma^t R_t | x]$ is the expected return given the path $x$. The goal in a reinforcement learning problem, in the context of the option-critic architecture, is to maximize the discounted return, $\rho(\mathcal O, \theta, \vartheta, s_0, o_0)$. The goal can be re-written as maximizing $J(\theta, s_0, o_0) = \sum \Pr(x; \theta) g(x)$. Where the summation is over all $x \in \mathcal{X}$ starting from $(s_0, o_0)$ and the intra-option policies are parametrized by $\theta$. To optimize the objective $J$, we define it over a Riemannian space $\Theta$, with $\theta \in \Theta$. In the Riemannian space the inner product is defined as in $\eqref{eq:rspacedist}$. The direction of steepest ascent of $J(\theta)$ in the Riemannian space, $\Theta$, is given by $G^{-1}_{\theta} \partial J(\theta) / \partial \theta$ \citep{Amari:1998:NGW:287476.287477}, (see equation \eqref{eq:natgrad_kakade}).

In this section we use $\partial_i$ to denote $\partial/\partial \theta_i$ and use $\langle f(X) \rangle_{\Pr(X)}$ to indicate the expected value of $f$ with respect to distribution $\Pr(X)$.
 We obtain an alternative form of the Fisher information matrix which is a well know result \citep{degroot2004optimal} (for details see appendix):
\begin{align}
\left (G_{\theta} \right )_{i,j} &= - \langle \partial_i \partial_j \ln \Pr(X; \theta) \rangle_{\Pr(X; \theta)} \label{eq:fimexpec}.
\end{align}
\subsection{Fisher Information Matrix Over Intra-Option Path Manifold}
In Theorem \ref{thm:1} we show that the Fisher information matrix over the paths, $X$, truncated to terminate at time step $\mathcal T$ converges as $\mathcal T \to \infty$ to the Fisher information matrix over the intra-option policies, $\pi_{o}$. This gives an expression for Fisher information matrix over the set of paths, $\mathcal X$, and simplifies computation of the natural gradient when maximizing the objective $J(\theta, s_0, o_0)$. We use $G^{\mathcal T}_{\theta}$ to indicate the ${\mathcal T}$-step finite horizon Fisher information matrix, meaning the Fisher information matrix if the problem were to be reduced to terminate at step $\mathcal T$.  We normalize the metric by the total length of path $\mathcal T$ \citep{bagnell2003covariant} to get a convergent metric. 
\begin{thm}[Infinite Horizon Intra-Option Matrix]\label{thm:1}
Let $G^{\mathcal T}_{\theta}$ be the $\mathcal T$-step finite horizon Fisher information matrix and $\langle G_{\theta} \rangle_{\mu_{\mathcal O}(s, o)}$ be the Fisher information matrix of intra-option policies under a stationary distribution of states, actions and options: $\pi_{o}(s,a,\theta) \mu_{\mathcal O}(s, o)$. Then:
\[\lim_{\mathcal T \to \infty} \frac{1}{\mathcal T} G^{\mathcal T}_{\theta} = \langle G_{\theta} \rangle_{\mu_{\mathcal O}(s, o)}.\]
\end{thm}
\begin{proof}
    See the appendix (supplementary materials).
\end{proof}
\subsection{Compatible Function Approximation For Intra-Option Path Manifold}
We subtract the option-state value function, $q_{\pi_{\mathcal O}}$, from the state-option-action value function, $q_U$, and treat it as a baseline to reduce variance in the gradient estimate of the expected discounted return. The baseline can be a function of both state and action in special circumstances, but none of those apply here \citep{DBLP:journals/corr/ThomasB17}. So, we  define the \textit{state-option-action advantage function} $a_U: \mathcal{S} \times \mathcal O \times \mathcal{A} \to \mathbb R$. Where $a_{U}(s, o, a) = q_{U}(s, o, a) - q_{\mathcal O}(s, o)$ is the advantage of the agent taking action $a$ in state $s$ in the context of option $o$. Here, $a_U$ is approximated by some compatible function approximator $f^{\pi_{o}}_{\eta}$. For vector $\eta$ and parameters $\theta$ we define:
\begin{equation}
    f_{\eta}^{\pi_{o}}(s,a) = \eta^T \left(\frac{\partial \ln (\pi_{o}(s, a, \theta))}{\partial \theta}\right). \label{eq:policy_func_approx}
\end{equation}
The $\tilde{\eta}$ that is a local minima of the squared error $\epsilon(\eta, \theta)$:
\[\epsilon(\eta, \theta) = \sum_{s, o, a} \mu_{\mathcal O} (s, o) \pi_{o}(s, a, \theta) (f_{\eta}^{\pi_{o}}(s,a) - a_{U}(s, o, a))^2.\]
is equal to the natural gradient of the objective, $\rho$, with respect to $\vartheta$  (the complete derivation is in the appendix):
\[ \widetilde{\nabla}_{\theta} q_{\pi_{\mathcal O}}(s_0, o_0) = G_{\theta}^{-1} \frac{\partial q_{\pi_{\mathcal O}}(s_0, o_0)}{\partial \theta} = \tilde{\eta}.\]
Thus, for a \textit{sensible} \citep{Kakade:2001} function approximation, as in \eqref{eq:policy_func_approx}, in the option-critic framework the natural gradient of the expected discounted return is the weights of linear function approximation.
\section{Start State Fisher Information Matrix Over State-Option Transition Path Manifold}
We derive the Fisher information matrix for the parameters $\vartheta$ over the state-option transitions path manifold. We define $X^\prime$ as a path for state-option transitions in the option-critic architecture. More specifically, we define $X^\prime = (O_0, S_1, O_1, S_2, O_2, S_3, ...)$ to be path tuples of state option pairs shifted by one time step. We define $\mathcal{X'}$ to be the set of all state-option transition paths. Similar to the previous section, we define the \textit{expected return over state-option transitions} $g':\mathcal{X'} \to \mathbb R$, where $g'(x') = \mathbb E[\sum_{t=0}^T \gamma^t R_t | x']$ is the expected return given state-option transitions path $x'$. The goal can be re-written to maximize $J'(\vartheta, s_1, o_0) = \sum \Pr(x') g'(x')$. Where the summation is over all $x' \in \mathcal{X'}$ starting from $(s_1, o_0)$ and terminations are parametrized by $\vartheta$. To optimize $J'$ we define it over a Reimannian space $\Theta'$ with $\vartheta \in \Theta'$ and the inner product defined as in \eqref{eq:rspacedist}, similar to previous section. The direction of steepest ascent in the Reimannian space, $\Theta'$, is the natural gradient.

In this section, we use $\partial_i $ to denote $\partial/\partial \vartheta_i$ and use $\langle f(X') \rangle_{\Pr(X')}$ to indicate the expected value of $f(X')$ with respect to the distribution $\Pr(X')$. Equation \eqref{eq:fimexpec} implies that the Fisher information matrix can be written as:
\[ \left (G_{\vartheta}\right)_{i,j} = - \langle \partial_i \partial_j \ln \Pr(X^\prime; \vartheta) \rangle_{\Pr(X^\prime; \vartheta)}.\]
\subsection{Fisher Information Matrix Over State-Option Transition Path Manifold}
In Theorem \ref{thm:2} we show that the Fisher information matrix over the paths, $X'$, truncated to terminate at time step $\mathcal T$ converges as $\mathcal T \to \infty$ to an expression in terms of the terminations and the policy over options over the stationary distribution of states and options. This gives an expression for Fisher information Matrix over set of paths, $\mathcal X'$, and simplifies computation of the natural gradient when maximizing the objective $J'(\vartheta, s_1, o_0)$.
\begin{thm}[Infinite Horizon State-Option Transition Matrix] Let $G^{\mathcal T}_{\vartheta}$ be the $\mathcal T$-step finite horizon Fisher information matrix and $\mu_{\mathcal O}(s', o)$ is the stationary distribution of state-option pairs $s', o$. Then:\label{thm:2}
\begin{align*}
    &\Big (\lim_{\mathcal T \to \infty} \frac{1}{\mathcal T} G^{\mathcal T}_{\vartheta}\Big)_{i,j} = - \langle  \partial_i \ln \beta_{o}(s^\prime, \vartheta) \\
    &\partial_j  \ln (1 - \beta_{o}(s^{\prime}, \vartheta) + \beta_{o}(s^{\prime}, \vartheta) \pi_{\mathcal O}(s^{\prime}, o)) \rangle_{\mu_{\mathcal O}(s', o)} .
\end{align*}
\end{thm}
\begin{proof}
See appendix (supplementary materials).
\end{proof}
\subsection{Compatible Function Approximation For State-Option Transition Path Manifold}
 We define the \textit{advantage function of continued option} as: $a'_{\mathcal O}: \mathcal{S} \times \mathcal O \to \mathbb R $. Where $a'_{\mathcal O}(s^{\prime}, o) = u(o, s^{\prime}) - q_{\pi_{\mathcal O}}(s^{\prime}, o)$ is the advantage of the option $o$ being active while exiting $s'$ given that option $o$ is active when the agent enters $s^{\prime}$. We consider terminations improvement when $a'_{\mathcal O}$ is approximated by some compatible function approximator $h^{\beta_{o}}_{\varphi}$. For vector $\varphi$ and parameters $\vartheta$ we define:
\begin{equation}
h^{\beta_{o}}_{\varphi}(s^{\prime})= \varphi^T  \frac{\partial \ln(1 - \beta_{o}(s^{\prime}, \vartheta) + \pi_{\mathcal O}(s^{\prime}, o)\beta_{o}(s^{\prime}, \vartheta)))}{\partial \vartheta}. \label{eq:func_approx_afco}
\end{equation}
 We define the squared error $\epsilon(\varphi, \vartheta)$ associated with vector $\varphi$ as:
\begin{align*}
\epsilon(\varphi, \vartheta) = \sum_{s^{\prime}, o} &\mu_{\mathcal O}(s^{\prime}, o)L(O_{t + 1} {=} o | O_t {=} o, S_{t + 1} {=} s^{\prime}; \vartheta)\\ &(h^{\beta_{o}}_{\varphi}(s^{\prime}) - a'_{\mathcal O}(s^{\prime}, o))^2,
\end{align*}
where $L(O_{t + 1} {=} o | O_t {=} o, S_{t + 1} = s^{\prime}; \vartheta)$ is the likelihood ratio of option $o$ being active while exiting $s'$ given that option $o$ is active when the agent enters $s^{\prime}$. It is defined as follows:
\begin{align*}
  L(O_{t + 1} {=} o  | O_t {=} o, S_{t + 1} &=s^{\prime}; \vartheta) \\  =&\frac{\Pr(O_{t + 1} {=} o | O_t {=} o, S_{t + 1} = s^{\prime}; \vartheta)}{\Pr(O_{t + 1} {\neq} o | O_t {=} o, S_{t + 1} = s^{\prime}; \vartheta)} \\
  =& \frac{\beta_{o}'(s', \vartheta)}{1 - \beta_{o}'(s', \vartheta)}.
\end{align*}
We assume, throughout the paper, that the denominator is not $0$. The $\tilde{\varphi}$ that is a local minima of $\epsilon(\varphi)$ satisfies (the complete derivation is in the appendix):
\[ \widetilde{\nabla}_{\vartheta} u(o_0, s_1) = G_{\vartheta}^{-1} \frac{\partial u(o_0, s_1)}{\partial \vartheta} = - \tilde{\varphi}.\]
Therefore, for an approximation of the continued state-option value function, as in \eqref{eq:func_approx_afco}, the natural gradient of the expected discounted return is the negative weights of the linear function approximation.
\section{Incremental Natural Option Critic Algorithm}
We introduce algorithms inspired from the incremental natural actor critic introduced by  \citet{Degris_model-freereinforcement}, who in turn built on the theoretical work of \citet{Bhatnagar:2007:INA:2981562.2981576}. The algorithm learns the parameters for approximations of state-option-action advantage function, $a_U$,  and the advantage function of continued option, $a_{\mathcal O}'$, incrementally by taking steps in the direction of reducing the error $\epsilon(\eta, 
\theta)$ and $\epsilon(\varphi, 
\vartheta)$. It does stochastic gradient descent using the gradients $\partial \epsilon(\eta, \vartheta)/\partial \eta$ and $\partial \epsilon(\varphi, \vartheta)/ \partial \varphi$. Learning the parameters $\eta$ and $\varphi$ leads to natural gradient based updates for $\theta$ and $\vartheta$. We introduce hyper parameters $\alpha_{\eta}, \alpha_{\varphi}$ and $\lambda$, which are the learning rate for $\eta$, the learning rate for $\varphi$ and the $\lambda$ the eligibility trace parameter of both $\eta$ and $\varphi$, respectively.  The algorithm learns the policy over options, $\pi_{\mathcal O}$, using intra-option Q-learning \citep{Sutton1999BetweenMA} as in previous work \citep{DBLP:journals/corr/BaconHP16}.

The algorithm uses \textit{TD-error} style updates to learn $\theta$ and $\vartheta$. Analogous to the consistent estimates used by \citet{Bhatnagar:2007:INA:2981562.2981576}, we state that a \textit{consistent estimate} of the state-option value function, $\hat{q}_{\pi_{\mathcal O}}$, satisfies $\mathbb E[\hat{q}_{\pi_{\mathcal O}}(s_t, o_t) | s_t, o_t, \pi_{\mathcal O}, \pi_{o_t}, \beta_{o_t} ] = q_{\pi_{\mathcal O}}$. Similarly, a consistent estimate of the value function upon arrival, $\hat{u}$, satisfies $\mathbb E[\hat{u}(o_t, s_{t + 1}) | o_t, s_{t+1}, \pi_{\mathcal O}, \pi_{o_t}, \beta_{o_t}] = u(o_t, s_{t + 1})$.  We define the TD-error for the intra-option policies at time step $t$ to be $\delta^U_t = r_t + \gamma \hat{u}(o_t, s_{t+1}) - \hat{q}_{\pi_{\mathcal O}}(s_{t}, o_t)$.

A consistent estimate of the state value function, $\hat{v}_{\pi_{\mathcal O}}$, satisfies $\mathbb E[\hat{v}_{\pi_{\mathcal O}}(s_{t})| s_t, \pi_{\mathcal O}, \pi_{o_t}, \beta_{o_t}] = v_{\pi_{\mathcal O}}(s_{t})$. We define the TD-error at time step $t$ for the terminations to be $\delta_t^{\mathcal O} = r_t + \gamma \hat{v}_{\pi_{\mathcal O}}(s_{t+1}) - \hat{v}_{\pi_{\mathcal O}}(s_{t})$. We provide Lemmas 1 and 2 to show that $\delta^U_t$ and $\delta^{\mathcal O}_t$ are consistent estimates of $a_{U}$ and $a_{\mathcal{O}}$.
\begin{lemma}
Given intra-option policies, $\pi_{o}$ for all $o \in \mathcal O$, policy over options, $\pi_{\mathcal O}$, and terminations, $\beta_{o}$ for all $o \in \mathcal O$, then: 
\[\mathbb E[\delta^U_t | s_t, a_t, o_t, \pi_{o_t}, \pi_{\mathcal O}, \beta_{o_t}] = a_{U}(s_t, o_t, a_t).\]
\end{lemma}
\begin{lemma}
Under the precondition $o_t = o_{t - 1}$ and given intra-option policies, $\pi_{o}$ for all $o \in \mathcal O$, policy over options, $\pi_{\mathcal O}$, and terminations, $\beta_{o}$ for all $o \in \mathcal O$, then:  \label{lemma:2}
\[\mathbb E[\delta^{\mathcal O}_t | s_{t}, o_t, o_{t} {=} o_{t-1}, \pi_{o_t}, \pi_{\mathcal O}] = a_{\mathcal O}(s_{t}, o_{t - 1}).\]
\end{lemma}
The proofs are in the appendix (supplementary materials). Using these lemmas and theorems we introduce algorithm \ref{algo:inoc} (INOC). We provide details on how we arrive at the updates to parameters $\eta$ and $\varphi$ in the appendix. The precondition $o_t = o_{t - 1}$ might lead to fewer updates to the parameters of the terminations. The options evaluation part in the algorithm is the same as in previous work \citep{DBLP:journals/corr/BaconHP16}.

\begin{algorithm}
\caption{Incremental Natural Option-Critic Algorithm (INOC)}
\begin{algorithmic}[1]
    \STATE  $s_0 \leftarrow d_0$ and choose $o$ using $\pi_{\mathcal O}$.
    \WHILE{Not in terminal state}
        \STATE Select action $a_t$ as per $\pi_{o_t}$
        \STATE Take action $a_t$ observe $s_{t+1}, r_{t}$ \\
        \STATE $e_{\eta} \leftarrow \lambda e_{\eta} + \frac{\partial \ln{\pi_{o_t}(s_{t}, a_{t}, \theta)}}{\partial \theta}$
        \STATE $\delta^U_t \leftarrow r_t + \gamma  u(o_t, s_{t+1}) - q_{\pi_{\mathcal O}}(s_{t}, o_{t})$
        \STATE $\texttt{temp} =  \frac{\partial \ln{\pi_{o_t}(s_{t}, a_{t}, \theta)}}{\partial \theta}$
        \STATE $\eta \leftarrow \eta + \alpha_{\eta} \delta^U_t  e_{\eta} - \alpha_{\eta} \texttt{temp} \times \texttt{temp}^T \times \eta $
        \STATE $\theta \leftarrow \theta + \alpha_{\theta} \frac{\eta}{||\eta||_2}$
        \IF{$o_t$ is the same as $o_{t-1}$}
            \STATE $e_{\varphi} \leftarrow \lambda e_{\varphi} + \frac{\partial \ln{\beta_{o_{t-1}}(s_{t}, \vartheta)}}{\partial \vartheta}$
            \STATE $\delta_t^{\mathcal O} \leftarrow r_t  + \gamma v_{\pi_{\mathcal O}}(s_{t + 1}) - \gamma v_{\pi_{\mathcal O}}(s_{t})$
            \STATE $\texttt{temp} = \frac{\partial \ln{\beta_{o_{t-1}}(s_{t}, \vartheta)}}{\partial \vartheta}$
            \STATE $\varphi \leftarrow \varphi + \alpha_{\varphi} \beta_{o_{t-1}}(s_{t}, \vartheta) \delta_t^{\mathcal O}  e_{\varphi} + \alpha_{\varphi} \texttt{temp} \times \texttt{temp}^T \times \varphi $
            \STATE $\vartheta \leftarrow \vartheta - \alpha_{\vartheta} \frac{\varphi}{||\varphi||_2}$
        \ENDIF
        \IF{should terminate $o_{t}$ in $s_{t+1}$ according to $\beta_{o_{t}}$}
            \STATE Choose $o_{t + 1}$ according to $\pi_{\mathcal O}$ and reset $\eta, \varphi, e_{\eta}, e_{\varphi}$
        \ENDIF
    \ENDWHILE \label{algo:inoc}
\end{algorithmic}
\end{algorithm}

\section{Experiments}
We look at the performance of natural option critic in three different types of domains: a simple 2 state MDP, one with linear state representations and one with neural networks for state representations, and compare it to option critic. In all the cases we use sigmoid terminations and linear-softmax intra-option policies, as in previous work \citep{DBLP:journals/corr/BaconHP16}.

\begin{figure}
    \centering
    \includegraphics[scale=0.4]{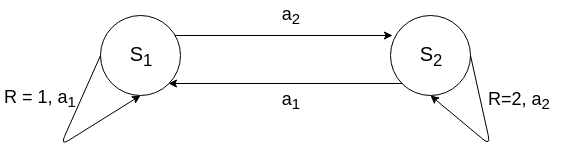}
    \caption{Simple deterministic MDP of two states and two actions}
    \label{fig:simplemdp}
\end{figure}

\begin{figure}
    \centering
    \includegraphics[scale=0.4]{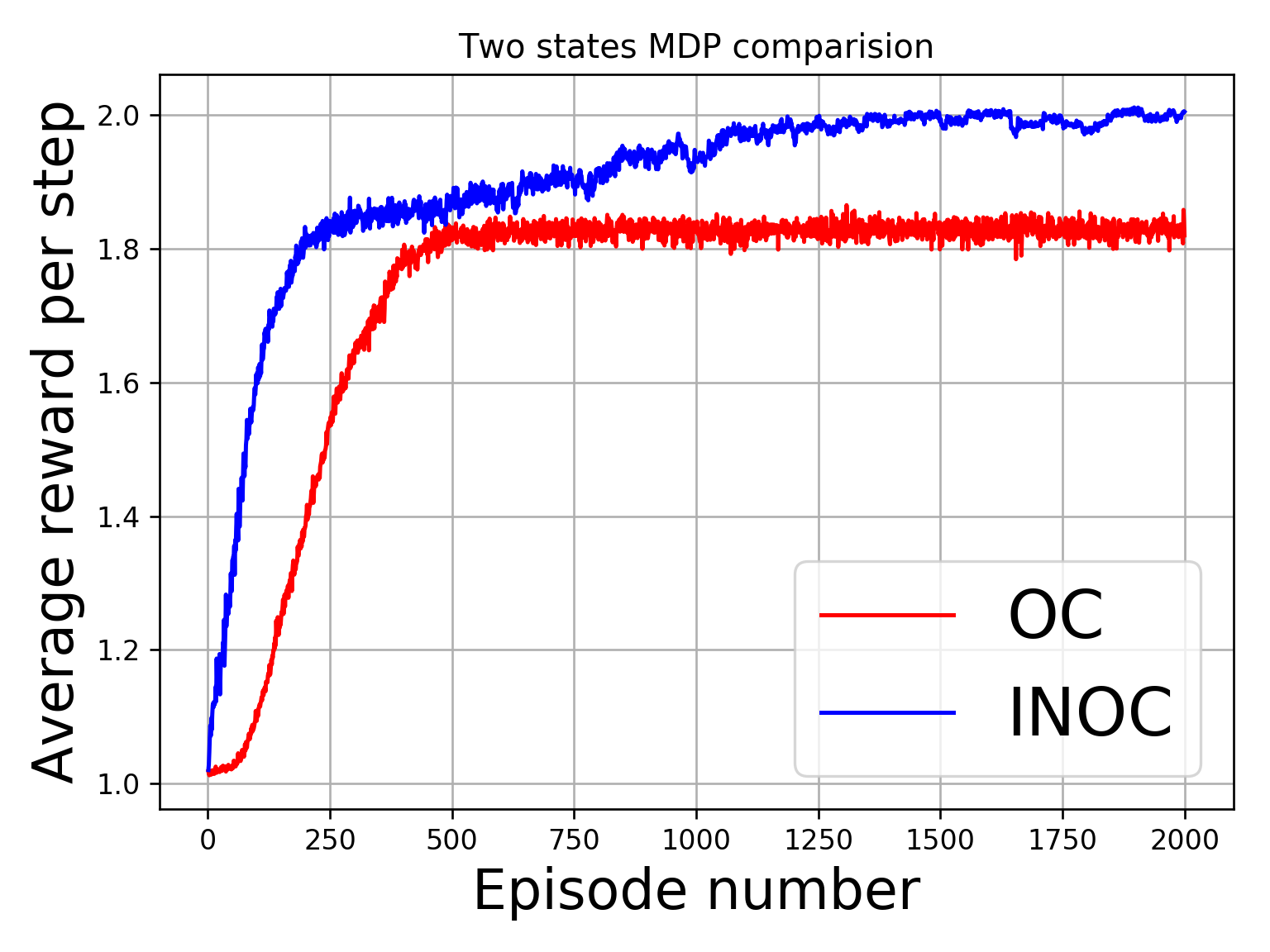}
    \caption{Average reward for INOC reaches the maxima while that of OC is stuck in a plateau. Results averaged over 200 runs of 2000 episodes.}
    \label{fig:inococsimple}
\end{figure}

{\bf MDP Setup: } %\subsection{Four Rooms}
We design an MDP to demonstrate the uniform weighting of the components of the natural termination gradient, $\widetilde{\nabla}_{\theta} q_{\pi_{\mathcal O}}(s_0, o_0)$, as opposed to using $\mu_{\mathcal O}(s, o)$. Note that the effectiveness of the natural policy gradient has been demonstrated sufficiently in past work (\citealp{Kakade:2001}; \citealp{bagnell2003covariant}; \citealp{Degris_model-freereinforcement}). We define a simple 2 state MDP as in Figure \ref{fig:simplemdp}. The initial state distribution is $d_0(s_1) = 0.8$ and $d_0(s_2) = 0.2$. The transitions are deterministic. The reward for self loops into $s_1$ and $s_2$ are 1 and 2, respectively. The episode terminates after 30 steps. We use an $\epsilon$-greedy policy over options, $\pi_{\mathcal O}$. 

We consider a scenario with two options, $o_1$ and $o_2$, each of which has probability 0.9 for actions $a_1$ and $a_2$, respectively, regardless of the state. This gives us options as abstractions over individual actions. We initialize the terminations, $\beta_o$, and option value function, $q_{\pi_{\mathcal O}}(s, o)$ such that they are biased towards the greedy action, $a_1$, in state $s_1$ via the selection of option $o_1$. Specifically, we set $\beta_{o_1}(s_1) = 0.1$ and $\beta_{o_1}(s_2) = 0.1$, this way the setup is biased towards higher probability of $\mu_{\mathcal O}(s_1, o_1)$. This presents the agent with the challenge of learning the more optimal action of transitioning to state $s_2$, despite the higher probability $\mu(s_1, o_1)$ and the self loop reward of $s_1$. We set the learning rate for the intra-option policies, $\alpha_{\theta}$, to be negligible as our goal is to demonstrate the efficacy of the natural termination gradient.

As can be seen from Figure \ref{fig:inococsimple}, the natural option critic converges to the optimal value, by overcoming the plateau, for average reward much faster than the option critic. The option critic is initially stuck in the greedy self-loop action, this is due to the weighting by $\mu_{\mathcal O}(s, o)$. Whereas the natural option critic begins learning early on and achieves the optimal average reward.

{\bf Four Rooms: } %\subsection{Four Rooms}
The four rooms domain \citep{Sutton1999BetweenMA} is a particularly favorable case for demonstrating the use of options. We use the same number of options, 4, as in previous work \citep{DBLP:journals/corr/BaconHP16}. The result (Figure \ref{fig:fourrooms_graph}) indicates that natural option critic converges faster.

{\bf Arcade Learning Environment: }%\subsection{Arcade Learning Environment}

\begin{figure*}[!tbp]
\centering
\begin{minipage}[b]{0.24\textwidth}
\includegraphics[width=.9\linewidth]{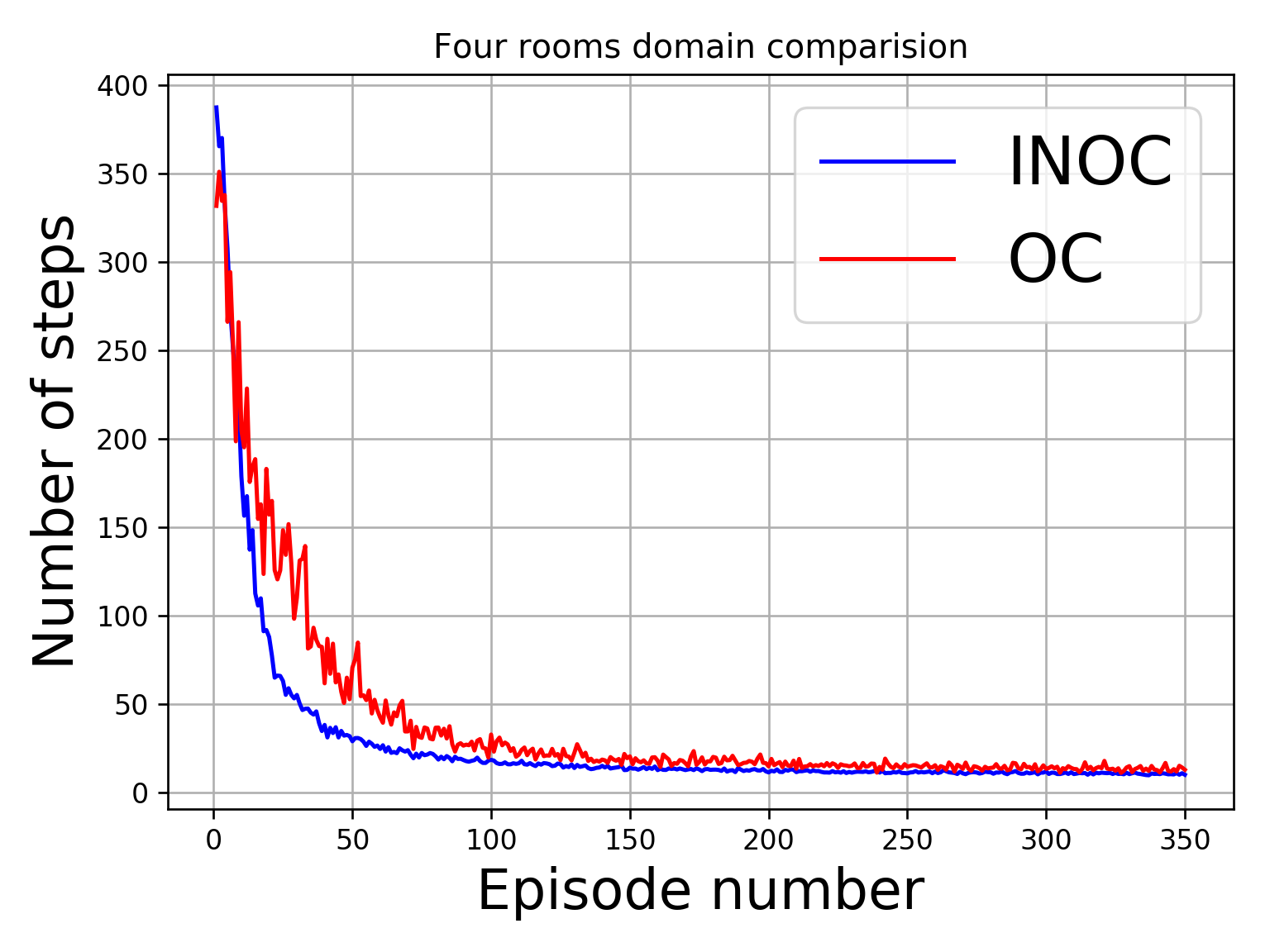}
\caption{Four rooms with $\alpha_{\theta} = \alpha_{\vartheta} = 0.0025$, $\alpha_{\eta} = 0.5$, $\alpha_{\varphi} =0.75$, $\lambda = 0.5$ and critic LR 0.5,  averaged over 350 runs }
\label{fig:fourrooms_graph}
\end{minipage}
\begin{minipage}[b]{0.74\textwidth}
\begin{subfigure}{.33\textwidth}
      \centering
      \includegraphics[width=.9\linewidth]{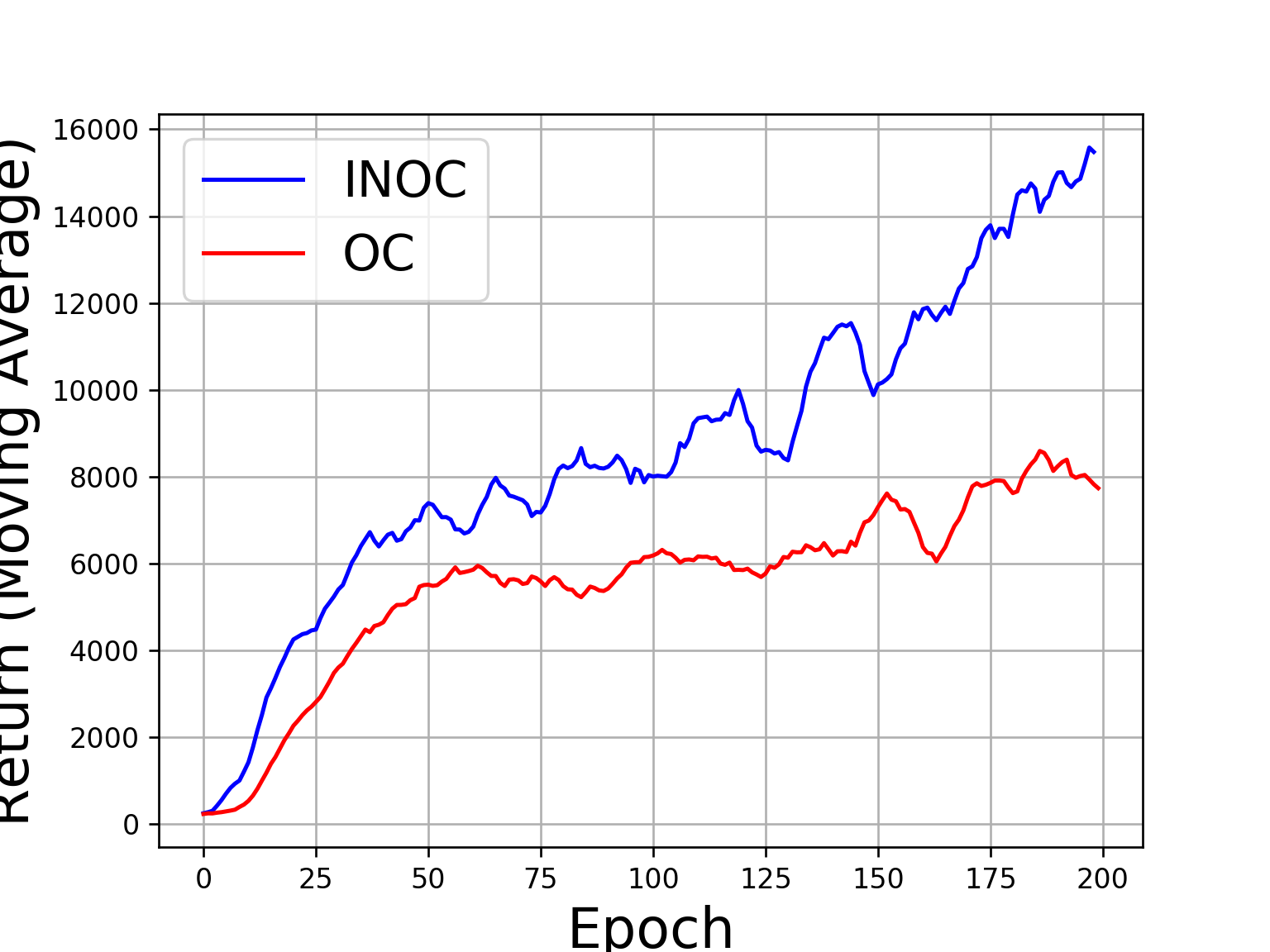}
      \caption{Asterisk}
      \label{fig:sfig1}
    \end{subfigure}%
    \begin{subfigure}{.33\textwidth}
      \centering
      \includegraphics[width=.95\linewidth]{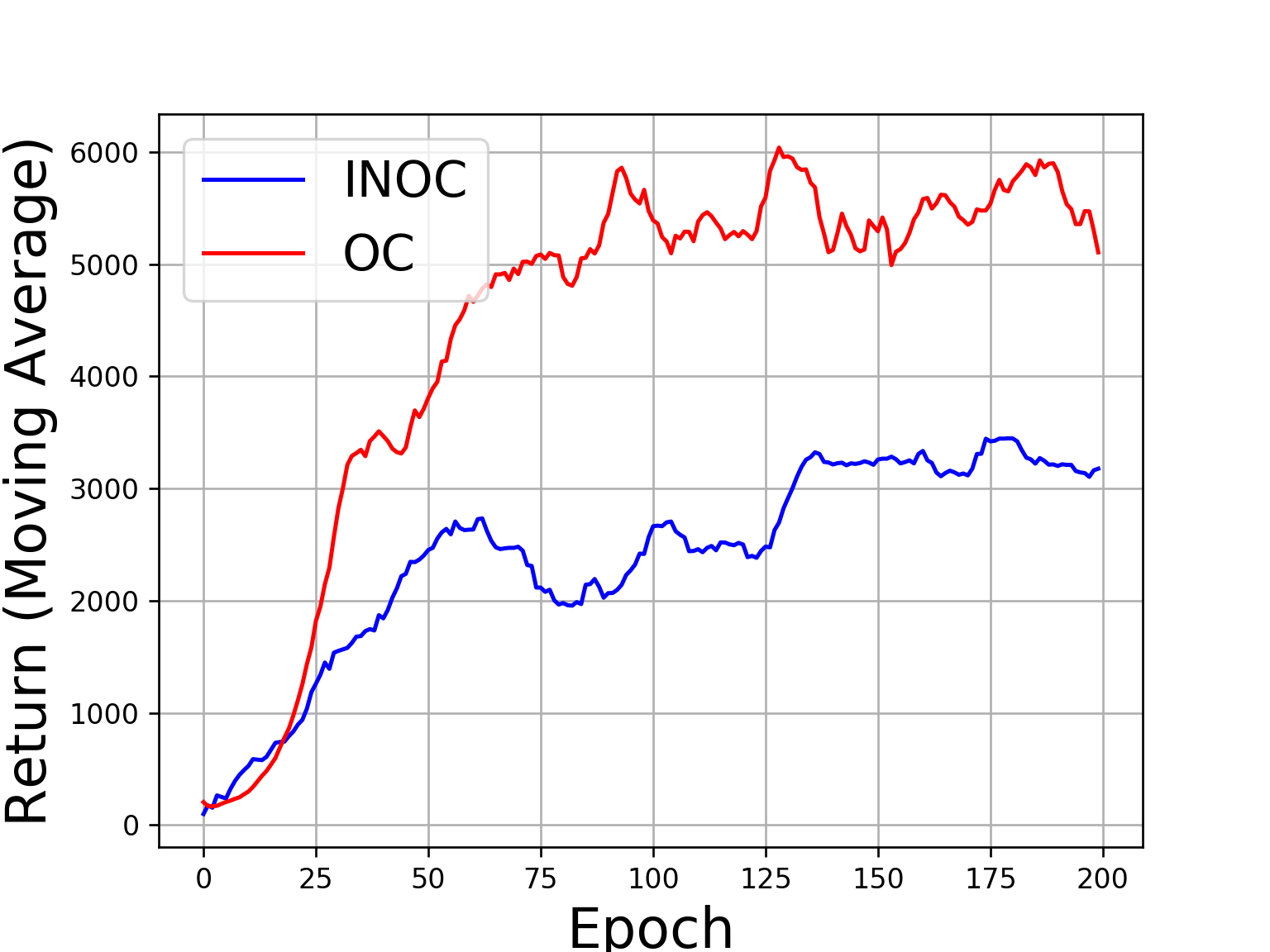}
      \caption{Seaquest}
      \label{fig:sfig2}
    \end{subfigure}
    \begin{subfigure}{.33\textwidth}
      \centering
      \includegraphics[width=.95\linewidth]{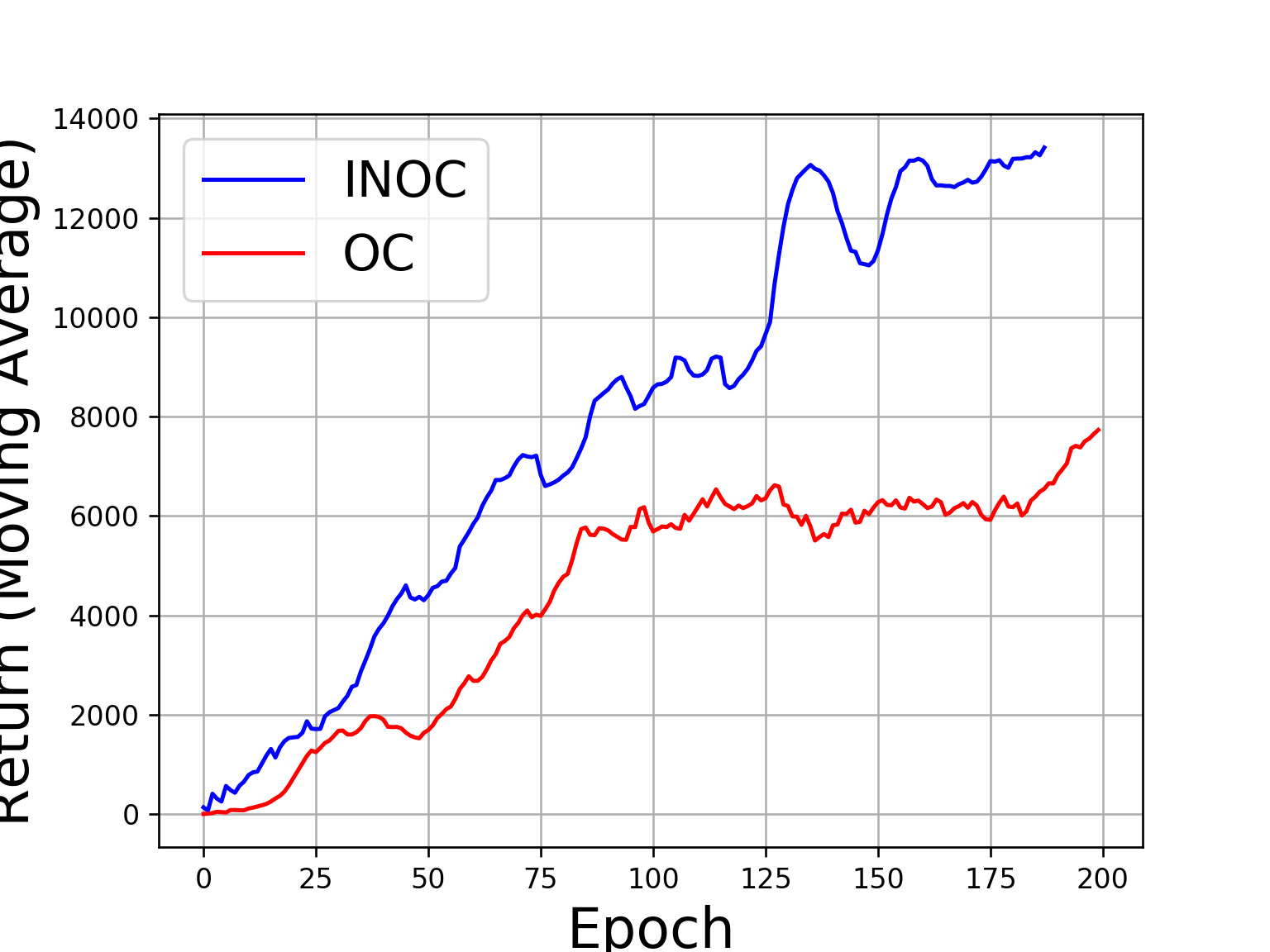}
      \caption{Zaxxon}
      \label{fig:sfig3}
    \end{subfigure}
    \caption{Moving average of 10 returns for a single trial for Arcade learning Environment, with  $\alpha_{\theta} = \alpha_{\vartheta} = 0.0025$, $\alpha_{\eta} = \alpha_{\varphi} = 0.75$, and $\lambda = 0.5$}
    \label{fig:graphs_ale}
\end{minipage}
\end{figure*}

We compare natural option-critic with the option critic framework on the Arcade Learning Environment \citep{Bellemare2013TheAL}. To showcase the improvement over the option-critic architecture we use the same configuration for all the layers as in previous work \citep{DBLP:journals/corr/BaconHP16}. Which in turn uses the same configuration for the first 3 convolutional layers of the network introduced by \citet{Mnih2013PlayingAW}. The critic network was trained, similar to previous work \citep{DBLP:journals/corr/BaconHP16}, using experience replay \citep{Mnih2013PlayingAW} and RMSProp.

As in previous work \citep{DBLP:journals/corr/BaconHP16}, we apply the regularizer prescribed by \citet{Mnih2016AsynchronousMF} to penalize low entropy policies. We use an on-policy estimate of the policy over options, $\pi_{\mathcal O}$, which is used in the computation of the natural gradient with respect to the termination parameters.

We compare the two approaches, option critic and natural option critic, by evaluating them for the games \textit{Asterisk, Seaquest,} and \textit{Zaxxon} \citep{DBLP:journals/corr/BaconHP16}. For comparison we run training over same number of frames per epoch as done by \citet{DBLP:journals/corr/BaconHP16}, running the same number of trial and use the same number of options: 8. We demonstrate the results in Figure \ref{fig:graphs_ale}. 
More importantly, we use the same hyperparameters, for learning rates and entropy regularization, as in previous work to merit a fair comparison. We obtain improvements on the option-critic architecture (OC) for Asterisk and Zaxxon. 
We also note that we were unable to reproduce the results for Seaquest for option critic, but having given the same set of hyperparameters we observe that option critic performs better. We explain the issue with termination updates, and it's effect on the return, for Seaquest in the appendix.

For Zaxxon and Asterisk we see that NOC breaks the plateau much earlier than option critic. Note that the value network, for approximating $Q_{U}$, is learned using vanilla gradient.

\section{Discussion}
We have introduced a natural gradient based approach for learning intra-option policies and terminations, within the option-critic framework, which is linear in the number of parameters. More importantly, we have furnished instructive proofs on deriving the Fisher information matrix over path manifolds and corresponding function approximations based approach while reducing mean squared errors. We have also introduced an algorithm that uses consistent estimates of the advantage functions and learn the natural gradient by learning coefficients of the corresponding linear function approximators. The results showcase performance improvements on previous work. The proofs for finite horizon metrics are very similar to the ones provided by \citet{bagnell2003covariant}. We  also demonstrate the effectiveness of natural option critic in three distinct domains. 

As discussed by \citet{Thomas2014BiasIN} we can obtain a truly unbiased estimate for our updates, but it may not be practical \citep{Thomas2014BiasIN}. The limitations that apply to the option-critic framework, except the use of vanilla gradient, apply. We use a block diagonal estimate of the Fisher information matrix. The complete Fisher information matrix for the option-critic framework over path manifolds is:
\[
G_{\theta, \vartheta} = \begin{bmatrix}
 G_{\theta} & \langle \frac{\partial X}{\partial \theta} \frac{\partial X}{\partial \vartheta}\rangle \\
\langle \frac{\partial X}{\partial \vartheta} \frac{\partial X}{\partial \theta}\rangle & G_{\vartheta}
\end{bmatrix},\]
where $G_{\theta}$ and $G_{\vartheta}$ are the Fisher information matrices for intra-option path manifold and state-option transition manifold, respectively. The random variable $X$ is the path variable over state-option-action tuples. The computation of the complete Fisher information matrix suffers and its inverse is expensive and needs a compatible function approximation based approach to obtain a natural gradient estimate with space complexity linear in number of parameters.

Although our approach has added benefits it is limited by fewer updates of the termination policy. Work is required to develop better estimates of the advantage functions. More experimental work, e.g. applications to other domains, can further help understand the efficacy of natural gradients in the context of the option-critic framework.
\small

\bibliography{aaai-2019}
\bibliographystyle{aaai} % We dont use APA in CS

\normalsize

\section*{Appendix}
Here, we provide proofs for the theorems and lemmas presented in the body of the paper and we also provide derivations for the estimates for the natural gradient. \textbf{Despite these proofs being in the appendix due to space constraint these are our major contributions}.
\subsection*{Alternate Form Of The Fisher Information Matrix}
We derive the following result, same as \citet{bagnell2003covariant} with the meanings of the symbols changed, for the Fisher information matrix under appropriate regularity conditions for $\mathcal{X}$:
\small
\begin{align}
\left ( G_{\theta} \right)_{i, j} =& \langle \partial_i \ln \Pr(X; \theta) \partial_j \ln \Pr(X; \theta)  \rangle_{\Pr(X)} \\
=& \sum_{X}  \partial_i \Pr(X; \theta) \partial_j \ln \Pr(X; \theta)\\
=& \sum_{X} \partial_i \Big (\Pr(X; \theta) \partial_j \ln \Pr(X; \theta) \Big)\\ 
&- \sum_{X} \Pr(X; \theta) \partial_i \partial_j \ln \Pr(X; \theta)\\
=& - \langle \partial_i \partial_j \ln \Pr(X; \theta) \rangle_{\Pr(X; \theta)} + \partial_i \partial_j \sum_{X} \Pr(X; \theta) \\
=& - \langle \partial_i \partial_j \ln \Pr(X; \theta) \rangle_{\Pr(X; \theta)}.
\end{align}
\normalsize

The first equality follows from the definition of Fisher information matrix. The third equality follows from integration by parts. The last equality is a result of the sum of probabilities being constant, i.e., $\sum_{X} \Pr(X) = 1$. The matrix $\langle \partial_i \ln \Pr(X) \partial_j \ln \Pr(X)  \rangle_{\Pr(X; \theta)}$ is positive semi-definite \citep{Amari1967ATO} and the derivations resulting from this expression inherit this property.
\subsection*{Proof Of Infinite Horizon Intra-Option Matrix}
\begin{thm*}[Infinite Horizon Intra-Option Matrix]
Let $G^{\mathcal T}_{\theta}$ be the $\mathcal T$-step finite horizon Fisher information matrix and $\langle G_{\theta} \rangle_{\mu_{\mathcal O}(s, o)}$ be the Fisher information matrix of intra-option policies under a stationary distribution of states, actions and options: $\pi_{o}(s,a,\theta) \mu_{\mathcal O}(s, o)$. Then:
\small
\[ \left ( \lim_{\mathcal T \to \infty} \frac{1}{\mathcal T} G^{\mathcal T}_{\theta} \right)_{i, j} = \langle G_{\theta} \rangle_{\mu_{\mathcal O}(s, o)}\]
\normalsize
\end{thm*}
\begin{proof}
$G^{\mathcal T}_{\theta}$ is the $\mathcal T $-step finite horizon Fisher information matrix.
\small
\begin{align}
\left ( \lim_{\mathcal T \to \infty} \frac{1}{\mathcal T} G^{\mathcal T}_{\theta} \right)_{i, j}
=& \lim_{T \to \infty} -\frac{1}{\mathcal T} \langle \partial_i \partial_j \ln \Pr(X) \rangle_{\Pr(X)} \\
=& \lim_{\mathcal T \to \infty} - \frac{1}{\mathcal T} \sum_{X} \Pr(X; \theta) \partial_i \frac{\partial_j \Pr(X; \theta)}{\Pr(X; \theta)} \label{eq:probratio}
\end{align}
\normalsize
The process represented by the path $X$ is Markovian, meaning $\Pr(S_t, O_t | S_{t - 1}, O_{t - 1}, S_{t - 2}, O_{t - 2}, ...) = \Pr(S_t, O_t | S_{t - 1}, O_{t - 1})$. This leads to the following result for the likelihood probability, similar to the simple form of the path probability metric presented by \citet{bagnell2003covariant}:
\small
\begin{align}
\frac{\partial_i \Pr(X; \theta)}{ \Pr(X; \theta)}  =& \partial_i \ln \Pr(X; \theta)
\\ =& \partial_i \ln \Pi_{t = 1}^{\mathcal T} \Pr(S_t, O_t | S_{t - 1}, O_{t - 1}; \theta)\\
=& \sum_{t = 1}^{\mathcal T} \partial_i \ln  \Pr(S_t, O_t | S_{t - 1}, O_{t - 1}; \theta) \\
=& \sum_{t = 1}^{\mathcal T} \frac{\partial_i \Pr(S_t, O_t | S_{t - 1}, O_{t - 1}; \theta)}{\Pr(S_t, O_t | S_{t - 1}, O_{t - 1}; \theta)}. \label{eq:markovsplit}
\end{align}
\normalsize
A reinforcement learning problem discounted with a discount factor $\gamma$ is equivalent to an undiscounted problem where the MDP terminates with probability $1 - \gamma $ in each state. We use this formulation of the problem to derive the results as we go further. Applying the chain rule to \eqref{eq:probratio} and using $\mu_{\mathcal O}(s, o)$ to denote the stationary start state distribution we obtain:
\small
\begin{align}
&\left ( \lim_{t \to \infty} \frac{1}{\mathcal T} G_{\theta} \right)_{i, j} \\
=& \lim_{\mathcal T \to \infty} - \frac{1}{\mathcal T} \Big \langle \sum_t  \Big (\frac{\partial_i \partial_j \Pr(S_t, O_t | S_{t - 1}, O_{t - 1}; \theta)}{\Pr(S_t, O_t | S_{t - 1}, O_{t - 1}; \theta)} \\
& - \frac{\partial_i \Pr(S_t, O_t | S_{t - 1}, O_{t - 1}; \theta) \partial_j \Pr(S_t, O_t | S_{t - 1}, O_{t - 1}; \theta)}{\Pr(S_t, O_t | S_{t - 1}, O_{t - 1}; \theta)^2} \Big ) \Big \rangle_{\Pr(X; \theta)} \\
=& - \sum_{o, s, a} \mu_{\mathcal O}(s, o) \pi_{\mathcal O} (s, a, \theta) \Big ( \frac{\partial_i \partial_j \pi_{o}(s, a, \theta)}{\pi_{o}(s, a, \theta)} \\
& - \frac{\partial_i \pi_{o}(s, a, \theta) \partial_j \pi_{o}(s, a, \theta)}{\pi_{o}(s, a, \theta)^2}  \Big ), \label{eq:decompose}
\end{align}
\normalsize
where $\mu_{\mathcal O}(s, o)$ is the probability of $(s, o)$ in the stationary distribution with the precondition $(s_0, o_0)$. The second equality above follows from the ergodic theorem \citep{stein2009real} and from the observation that the terms in the numerator and the denominator cancel out as follows:
\small
\begin{align}
&\frac{\partial_i \Pr(S_t, O_t | S_{t - 1}, O_{t - 1}, a; \theta)}{\Pr(S_t, O_t | S_{t - 1}, O_{t - 1}, a; \theta)} \\
=&\frac{\partial_i \pi_{O_t}(S_{t - 1}, a, \theta) \gamma \Pr(S_t| S_{t - 1}, a) \Pr(O_t | O_{t-1}, S_t)}{\pi_{O_t}(S_{t-1}, a, \theta) \gamma \Pr(S_t| S_{t - 1}, a) \Pr(O_t | O_{t-1}, S_t)} \\
= &\frac{\partial_i \pi_{O_t}(S_{t - 1}, a, \theta)}{\pi_{O_t}(S_{t - 1}, a, \theta)},
\end{align}
\normalsize
where $\Pr(O_t | O_{t-1}, S_t) =(1 - \beta_{O_{t-1}}(S_{t -1}) )\textbf{1}_{O_{t - 1} {=} O_t} + \beta_{O_{t - 1}}(S_{t})\pi_{\mathcal{O}}(O_t, S_t)$ is the probability of option $O_t$ being active while exiting $S_t$ given that the option $O_{t - 1}$ is active when the agent enters $S_t$. Continuing the derivation from \eqref{eq:decompose}:
\small
\begin{align}
\Big ( \lim_{T \to \infty} & \frac{1}{T} G^T_{\theta} \Big )_{i, j} \\
=& \sum_{o, s, a} \mu_{\mathcal O}(s, o) \pi_{o} (s, a, \theta) \frac{\partial_i \pi_{o}(s, a, \theta) \partial_j \pi_{o}(s, a, \theta)}{\pi_{o}(s, a, \theta)^2} \\
&- \sum_{o, s} \mu_{\mathcal O}(s, o) \sum_{a} \partial_i \partial_j \pi_{o}(s, a, \theta) \\
=& \sum_{o, s, a} \mu_{\mathcal O}(s, o) \pi_{o} (s, a, \theta) \partial_i \ln \pi_{o}(s, a, \theta) \partial_j \ln \pi_{o}(s, a, \theta) \\
=& \langle G_{\theta} \rangle_{\mu_{\mathcal O}(s, o)}.
\end{align}
\normalsize

The second term in the first equality above vanishes because $\sum_{a} \pi_{o}(s, a, \theta)$ is constant.
\end{proof}
\subsection*{Derivation Of Compatible Function Approximation For Intra-Option Path Manifold}
Given the state-option-action advantage function, $a_U$, and its approximator $f^{\pi_{o}}_{\eta}$. For vector $\eta$ and parameters $\theta$ we have:
\small
\begin{align*}
    f_{\eta}^{\pi_{o}}(s,a) &= \eta^T \left(\frac{\partial \ln (\pi_{o}(s, a, \theta))}{\partial \theta}\right).
\end{align*}
\normalsize
Let $\tilde{\eta}$ minimize the squared error $\epsilon(\eta, \theta)$:
\[\epsilon(\eta, \theta) = \sum_{s, o, a} \mu_{\mathcal O} (s, o) \pi_{o}(s, a, \theta) (f_{\eta}^{\pi_{o}}(s,a) - a_{U}(s, o, a))^2,\] 
therefore, it satisfies $\partial \epsilon/\partial \eta = 0$:
\small
\begin{align*}
    &\sum_{s, o, a} \mu_{\mathcal O} (s, o) \pi_{o}(s, a, \theta) \frac{\partial \ln \pi_{o}(s, a, \theta)}{\partial \theta} (f_{\tilde{\eta}}^{\pi_{o}}(s,a) - a_{U}(s, o, a)) = 0 \\
    &\sum_{s, o, a} \mu_{\mathcal O} (s, o) \pi_{o}(s, a, \theta) \frac{\partial \ln \pi_{o}(s, a, \theta)}{\partial \theta} \frac{\partial \ln \pi_{o}(s, a, \theta)}{\partial \theta}^T \tilde{\eta} \\
    &= \sum_{s, o, a} \mu_{\mathcal O} (s, o) \pi_{o}(s, a, \theta) a_{U}(s, o, a).
\end{align*}
\normalsize
Combining this with the intra-option policy gradient theorem \citep{DBLP:journals/corr/BaconHP16} we get:
\small
\begin{align*}
    \sum_{s, o, a} \mu_{\mathcal O} (s, o) \pi_{o}(s, a, \theta) \frac{\partial \ln \pi_{o}(s, a, \theta)}{\partial \theta}& \frac{\partial \ln \pi_{o}(s, a, \theta)}{\partial \theta}^T \tilde{\eta} \\
    =& \frac{\partial q_{\pi_{\mathcal O}}(s_0, o_0)}{\partial \theta}.
\end{align*}
\normalsize
Finally, using Theorem 1 we obtain an estimate for the natural gradient of the expected discounted return:
\[ \widetilde{\nabla}_{\theta} q_{\pi_{\mathcal O}}(s_0, o_0) = G_{\theta}^{-1} \frac{\partial q_{\pi_{\mathcal O}}(s_0, o_0)}{\partial \theta} = \tilde{\eta}.\]
\subsection*{Proof Of Infinite Horizon State-Option Transition Matrix}
\begin{thm*}[Infinite Horizon State-Option Transition Matrix] Let $G^{\mathcal T}_{\vartheta}$ be the $\mathcal T$-step finite horizon Fisher information matrix and $\mu_{\mathcal O}(s', o)$ is the stationary distribution of state-option pair $s', o$. Then:
\small
\begin{align*}
&\left ( \lim_{\mathcal T \to \infty} \frac{1}{\mathcal T} G^{\mathcal T}_{\vartheta} \right)_{i, j}\\ 
&= - (\langle  \partial_i \ln \beta_{o}(s^\prime, \vartheta) \partial_j  \ln (1 - \beta_{o}(s^{\prime}, \vartheta) + \beta_{o}(s^{\prime}, \vartheta) \pi_{\mathcal O}(s^{\prime}, o)) \rangle_{\mu_{\mathcal O}(s', o)})_{i,j}
\end{align*}
\normalsize
\end{thm*}
\begin{proof}
\small
\begin{align}
&\left ( \lim_{\mathcal T \to \infty} \frac{1}{\mathcal T} G^{\mathcal T}_{\vartheta} \right)_{i, j} 
= \lim_{\mathcal T \to \infty} -\frac{1}{\mathcal T} \langle \partial_i \partial_j \ln \Pr(X^\prime; \vartheta) \rangle_{\Pr(X^\prime; \vartheta)} \\
&= \lim_{\mathcal T \to \infty} - \frac{1}{\mathcal T} \Big \langle \sum_t \Big ( \frac{\partial_{i} \partial_j \Pr(S_{t +1}, O_t | S_t, O_{t - 1}; \vartheta)}{ \Pr(S_{t +1}, O_t | S_t, O_{t - 1}; \vartheta)} \\
& - \frac{\partial_i  \Pr(S_{t +1}, O_t | S_t, O_{t - 1}; \vartheta) \partial_j  \Pr(S_{t +1}, O_t | S_t, O_{t - 1}; \vartheta)}{ \Pr(S_{t +1}, O_t | S_t, O_{t - 1}; \vartheta)^2} \Big ) \Big \rangle_{\Pr(X^\prime; \vartheta)}. \label{eq:termcomplete}
\end{align}
\normalsize
The second equality follows from the simplification as in \eqref{eq:markovsplit} which is based on the fact that option transitions are Markovian, meaning $\Pr(S_t, O_{t -1} | S_{t - 1}, O_{t - 2}, S_{t - 2}, O_{t - 3}, ...; \vartheta) = \Pr(S_t, O_{t -1} | S_{t - 1}, O_{t - 2}; \vartheta) $. Before moving forward we first note that:
\small
\begin{align}
    &\Pr(S_{t + 1}, O_t | S_t, O_{t - 1}; \vartheta) \\
    =& \left ((1 - \beta_{O_{t -1}}(S_{t}, \vartheta)) \textbf{1}_{O_t = O_{ t -1}}  + \beta_{O_{t - 1}}(S_t, \vartheta) \pi_{\mathcal O}(S_t, O_t) \right ) \\
    &\times \left(\sum_{a} \pi_{O_t}(S_t, a)\gamma \Pr(S_{t+1}| S_t, a)\right) \label{eq:op_tran_prob_exp},
\end{align}
\normalsize
where $\times$ denotes scalar multiplication and $\Pr(S_{t + 1}, O_t | S_t, O_{t - 1}; \vartheta)$ is the probability of the agent transitioning to $(S_{t+1}, O_t)$ given that option $O_{t - 1}$ is active when the agent enters $S_t$. Expanding the first part of the expression in \eqref{eq:termcomplete}:
\small
\begin{align}
&\frac{\partial_i \partial_j \Pr(S_{t + 1}, O_t | S_t, O_{t - 1}; \vartheta)}{ \Pr(S_{t + 1}, O_t | S_t, O_{t - 1}; \vartheta)}  \\
&= \frac{\partial_i \partial_j \left ((1 - \beta_{O_{t -1}}(S_{t}, \vartheta)) \textbf{1}_{O_t {=} O_{ t -1}}  + \beta_{O_{t - 1}}(S_t, \vartheta) \pi_{\mathcal O}(S_t, O_t) \right)}{ \left ((1 - \beta_{O_{t -1}}(S_{t}, \vartheta)) \textbf{1}_{O_t {=} O_{ t{-}1}}  + \beta_{O_{t - 1}}(S_t, \vartheta) \pi_{\mathcal O}(S_t, O_t) \right)} \\
&= \frac{\partial_i \partial_j \Pr(O_{t} | S_{t}, O_{t - 1}; \vartheta)}{\Pr(O_{t} | S_{t}, O_{t - 1}; \vartheta)} \label{eq:ll_diff}.
\end{align}
\normalsize
Where,  $\Pr(O_{t} | S_{t}, O_{t - 1}; \vartheta) = (1 - \beta_{O_{t -1}}(S_{t}, \vartheta)) \textbf{1}_{O_t {=} O_{ t -1}}  + \beta_{O_{t - 1}}(S_t, \vartheta) \pi_{\mathcal O}(O_t | S_t)$ is the probability that the agent transitions to option $O_t$ given that the option $O_{t - 1}$ is active as it entered $S_t$. Expanding the first term in \eqref{eq:termcomplete} using \eqref{eq:ll_diff} we obtain:
\small
\begin{align}
\lim_{T \to \infty}& - \frac{1}{T} \langle \sum_t \frac{\partial_i \partial_j \Pr(S_{t + 1}, O_t | S_t, O_{t - 1}; \vartheta)}{ \Pr(S_{t + 1}, O_t | S_t, O_{t - 1}; \vartheta)} \rangle_{\Pr(X^\prime; \vartheta)}\\
=& \lim_{T \to \infty} - \frac{1}{T} \langle \sum_t \frac{\partial_i \partial_j \Pr(O_{t} | S_{t}, O_{t - 1}; \vartheta)}{\Pr(O_{t} | S_{t}, O_{t - 1}; \vartheta)} \rangle_{\Pr(X^\prime; \vartheta)}\\
=& \sum_{s^{\prime}, o} \mu_{\mathcal O}(s^{\prime}, o) \Big ( \Pr(O_t \neq o | O_{t - 1} = o, S_t = s') \frac{\partial_i \partial_j  \beta_{o}(s', \vartheta)}{\beta_{o}(s', \vartheta)} \\
& + \Pr(O_t = o | O_{t - 1} = o, S_t = s') \\
& \frac{\partial_i \partial_j (1 - \beta_o(s', \vartheta) + \beta_o(s', \vartheta) \pi_{\mathcal O}(s', o))}{(1 - \beta_o(s', \vartheta) + \beta_o(s', \vartheta) \pi_{\mathcal O}(s', o))} \Big )\\
=& \sum_{s^{\prime}, o} \mu_{\mathcal O}(s^{\prime}, o) \Big ( \beta_{o}(s^\prime, \vartheta)(1 - \pi_{\mathcal O}(s^\prime, o))\frac{\partial_i \partial_j \beta_{o}(s^\prime, \vartheta)}{\beta_{o}(s^\prime, \vartheta)} \\
& + \left (1 - \beta_{o}(s^{\prime}, \vartheta) + \beta_{o}(s^{\prime}, \vartheta) \pi_{\mathcal O}(s^\prime, o) \right)\\
&\frac{(\pi_{\mathcal O}(s^\prime, o) - 1)\partial_i \partial_j \beta_{o}(s^\prime, \vartheta)}{\left (1 - \beta_{o}(s^\prime, \vartheta)  + \beta_{o}(s^\prime, \vartheta) \pi_{\mathcal O}(s^\prime, o) \right)}  \Big )\\
=& \sum_{s^{\prime}, o} \mu_{\mathcal O}(s^{\prime}, o) \left (1 - \pi_{\mathcal O}(s^\prime) + \pi_{\mathcal O}(s^\prime) - 1  \right ) \partial_i \partial_j \beta_{o}(s^\prime, \vartheta) = 0\label{eq:termfirstexpr}.
\end{align}
\normalsize

The second equality above follows from the Ergodic theorem \citep{stein2009real}. We define $\beta'_{o}$ as the distribution of continuing option, where $\beta_o'(s, \vartheta)$ (as opposed to $\beta_o(s, o)$) is the probability, parametrized by $\vartheta$, that the option $o$ is active while the agent is exiting $s'$ given that option $o$ is active when it enters $s'$. It is given by $\beta'_{o}(s', \vartheta) = (1 - \beta_o(s', \vartheta) + \beta_o(s', \vartheta) \pi_{\mathcal O}(s', o))$. We now evaluate the second term in \eqref{eq:termcomplete} and continue the derivation:
\small
\begin{align}
\lim_{T \to \infty} &\frac{1}{T} \langle \sum_t \frac{\partial_i  \Pr(S_{t +1}, O_t | S_t, O_{t - 1}; \vartheta) \partial_j  \Pr(S_{t +1}, O_t | S_t, O_{t - 1}; \vartheta)}{ \Pr(S_{t +1}, O_t | S_t, O_{t - 1}; \vartheta)^2} \rangle_{\Pr(X^{\prime}; \vartheta)}\\
=& \sum_{s^\prime, o} \mu_{\mathcal O}(s^\prime, o) \Big ( \Pr(O_t \neq o | O_{t - 1} = o, S_t = s') \frac{\partial_i \beta_{o}(s', \vartheta) \partial_j  \beta_{o}(s', \vartheta)}{\beta_{o}(s', \vartheta)^2}   \\
& + \Pr(O_t = o | O_{t - 1} = o, S_t = s') \frac{\partial_i \beta'_{o}(s', \vartheta) \partial_j \beta'_{o}(s', \vartheta)}{\beta'_{o}(s', \vartheta)^2} \Big ) \\
=& \sum_{s^\prime, o} \mu_{\mathcal O}(s^\prime, o) \Big (\beta_{o}(s^\prime, \vartheta)(1 - \pi_{\mathcal O}(s^{\prime}, o)) \frac{\partial_i \beta_{o}(s^\prime, \vartheta) \partial_j \beta_{o}(s^\prime, \vartheta)}{\beta_{o}(s^\prime, \vartheta)^2}  \\
& + \beta'_{o}(s', \vartheta) \frac{(1 - \pi_{\mathcal O}(s^{\prime}, o))^2 \partial_i \beta_{o}(s^\prime, \vartheta) \partial_j \beta_{o}(s^\prime, \vartheta)}{\left (1 - \beta_{o}(s^{\prime}, \vartheta) + \beta_{o}(s^{\prime}, \vartheta) \pi_{\mathcal O}(s^{\prime}, o) \right)^2} \Big)\\
=& \sum_{s^\prime, o} \mu_{\mathcal O}(s^\prime, o) \left ( \frac{(1 - \pi_{\mathcal O}(s^{\prime}, o)) \partial_i \beta_{o}(s^\prime, \vartheta) \partial_j \beta_{o}(s^\prime, \vartheta)}{\beta_{o}(s^\prime, \vartheta) \left (1 - \beta_{o}(s^{\prime}, \vartheta) + \beta_{o}(s^{\prime}, \vartheta) \pi_{\mathcal O}(s^{\prime}, o) \right)} \right )\\
=& \sum_{s^\prime, o} - \mu_{\mathcal O}(s^\prime, o)  \partial_i \ln \beta_{o}(s^\prime, \vartheta) \partial_j  \ln (1 - \beta_{o}(s^{\prime}, \vartheta) + \beta_{o}(s^{\prime}, \vartheta) \pi_{\mathcal O}(s^{\prime}, o)).
\end{align}
\normalsize
The first equality follows from the ergodic theorem \citep{stein2009real} and \eqref{eq:op_tran_prob_exp}. The third equality is a result of arithmetic simplification. Combining this result with \eqref{eq:termfirstexpr} we obtain:
\small
\begin{align} 
    \left ( G_{\vartheta} \right )_{i,j} =& - \sum_{s^\prime, o} \mu_{\mathcal O}(s^\prime, o) \frac{\partial \ln \beta_{o}(s^\prime, \vartheta)}{\partial \vartheta} \\
    &\left (\frac{\partial \ln (1 - \beta_{o}(s^\prime, \vartheta) + \beta_{o}(s^\prime, \vartheta) \pi_{\mathcal O}(s^\prime, o))}{\partial \vartheta}\right)^T. \label{eq:natgradterm}
\end{align}
\normalsize
\end{proof}
\subsection*{Derivation Of Compatible Function Approximation For Option Termination Path Manifold}
Given the advantage function of continued option, $a'_{\mathcal O}$, and its compatible function approximation $h^{\beta_{o}}_{\varphi}$:
\begin{align}
h^{\beta_{o}}_{\varphi}(s^{\prime})= &\left (\frac{\partial \ln(1 - \beta_{o}(s^{\prime}, \vartheta) + \pi_{\mathcal O}(s^{\prime}, o)\beta_{o}(s^{\prime}, \vartheta)))}{\partial \vartheta} \right)^T \\ 
& \phi^{\beta_{o}}(s^{\prime}, o). \label{eq:func_approx_afco_2}
\end{align}
Note that $\Pr(O_{t + 1} {=} o | O_t {=} o, S_{t + 1} {=} s^{\prime}; \vartheta) = 1 - \beta_{o}(s^{\prime}, \vartheta) + \pi_{\mathcal O}(s^{\prime}, o)\beta_{o}(s^{\prime}, \vartheta) = \beta'_{o}(s', \vartheta) $, follows from the definitions above. We defined the mean squared error $\epsilon(\varphi, \vartheta)$ associated with vector $\varphi$ as:
\begin{align*}
\epsilon(\varphi, \vartheta) = \sum_{s^{\prime}, o} &\mu_{\mathcal O}(s^{\prime}, o)L(O_{t + 1} {=} o | O_t {=} o, S_{t + 1} {=} s^{\prime}; \vartheta) (h^{\beta_{o}}_{\varphi}(s^{\prime}) - a'_{\mathcal O}(s^{\prime}, o))^2,
\end{align*}
where $L(O_{t + 1} {=} o | O_t {=} o, S_{t + 1} = s^{\prime}; \vartheta)$ is the likelihood ratio of option $o$ being continued given that we enter state $s^{\prime}$ in $o$. It is defined as follows:
\begin{align*}
  L(O_{t + 1} {=} o  |& O_t {=} o, S_{t + 1} =s^{\prime}; \vartheta)  \\
  =& \frac{\Pr(O_{t + 1} {=} o | O_t {=} o, S_{t + 1} = s^{\prime}; \vartheta)}{\Pr(O_{t + 1} {\neq} o | O_t {=} o, S_{t + 1} = s^{\prime}; \vartheta)}\\
  =& \frac{1 - \beta_{o}(s^{\prime}, \vartheta) + \pi_{\mathcal O}(s^{\prime}, o)\beta_{o}(s^{\prime}, \vartheta)}{\beta_{o}(s', \vartheta) (1 - \pi_{\mathcal O}(s'))}
\end{align*}
 We also note the following result related to the likelihood ratio, $L$, and the partial derivative of the probability of continuing option, $\beta'_{o}$, before proceeding:
\begin{align}
    L(O_{t + 1} {=} o |& O_t {=} o, S_{t + 1} {=} s^{\prime}; \vartheta) \frac{\partial \ln \beta'_{o}(s', o)}{\partial \vartheta} \\
    =& \frac{1 - \beta_{o}(s^{\prime}, \vartheta) + \pi_{\mathcal O}(s^{\prime}, o)\beta_{o}(s^{\prime}, \vartheta)}{\beta_{o}(s', o)(1 - \pi_{\mathcal O}(s', o))} \\
    &\frac{\partial \ln 1 - \beta_{o}(s^{\prime}, \vartheta) + \pi_{\mathcal O}(s^{\prime}, o)\beta_{o}(s^{\prime}, \vartheta)}{\partial \vartheta} \\
    =& \frac{1}{\beta_{o}(s', o)(1 - \pi_{\mathcal O}(s', o))}\\
    &\frac{\partial (1 - \beta_{o}(s^{\prime}, \vartheta) + \pi_{\mathcal O}(s^{\prime}, o)\beta_{o}(s^{\prime}, \vartheta))}{\partial \vartheta} \\
    =& - \frac{\partial \beta_{o}(s', \vartheta)}{ \beta_{o}(s', \vartheta) \partial \vartheta} = - \frac{\partial \ln \beta_o(s', \vartheta)}{\partial \vartheta}. \label{eq:lr_beta_prime_rel}
\end{align} 
Let $\tilde{\varphi}$ be a local minima of the expected squared error, $\epsilon(\varphi, \vartheta)$. Therefore, it satisfies $\partial \epsilon /\partial \varphi = 0$. So,
\small
\begin{align*}
    &\sum_{s^{\prime}, o} \mu_{\mathcal O}(s^{\prime}, o)L(O_{t + 1} {=} o | O_t {=} o, S_{t + 1} {=} s^{\prime}; \vartheta)\\
    &\times (\phi^{\beta_{o}}(s^{\prime}, o))^T (h^{\beta_{o}}_{\tilde{\varphi}}(s^{\prime}) - a'_{\mathcal O}(s^{\prime}, o)) = 0  \\
    &- \sum_{s^{\prime}, o}  \mu_{\mathcal O}(s^{\prime}, o) \frac{ \partial \ln{\beta_{o}(s^{\prime}, \vartheta)}} {\partial \vartheta} (h^{\beta_{o}}_{\tilde{\varphi}}(s^{\prime}) - a'_{\mathcal O}(s^{\prime}, o)) = 0 \\
    &\sum_{s^{\prime}, o} \mu_{\mathcal O}(s^{\prime}, o) \frac{\partial \ln{\beta_{o}(s^{\prime}, \vartheta)}} {\partial \vartheta} \\
    &\left( \frac{\partial\ln (1 - \beta_{o}(s^{\prime}, \vartheta) + \pi_{\mathcal O}(s^{\prime}, o)\beta_{o}(s^{\prime}, \vartheta))}{\partial \vartheta} \right)^T \tilde{\varphi} = \\ 
    &\sum_{s^{\prime}, o} \mu_{\mathcal O}(s^{\prime}, o) a'_{\mathcal O}(s^{\prime}, o) \frac{\partial  \ln{\beta_{o}(s^{\prime}, \vartheta)}}{\partial \vartheta}.
\end{align*}
\normalsize
The second line above follows from \eqref{eq:lr_beta_prime_rel}. To proceed further we observe from:
\begin{equation}
    \label{eq:onarrival_2}
    u(o, s^{\prime}) = (1 - \beta_{o}(s^{\prime}))q_{\pi_{\mathcal O}}(s^{\prime}, o) + \beta_{o}(s^{\prime})v_{\pi_{\mathcal O}}(s^{\prime}).
\end{equation}
we have $a'_{\mathcal O}(s^{\prime}, o) = u(o, s^{\prime}) - q_{\mathcal O}(s^{\prime}, o) = -\beta_{o}(s^{\prime}) a_{\mathcal O}(s^{\prime}, o)$. Combining this with the termination gradient theorem \citep{DBLP:journals/corr/BaconHP16} we obtain:
\begin{align*}
    &\sum_{s^{\prime}, o} \mu_{\mathcal O}(s^{\prime}, o) \frac{ \partial \ln{\beta_{o}(s^{\prime}, \vartheta)}} {\partial \vartheta} \\
    & \left( \frac{\partial\ln (1 - \beta_{o}(s^{\prime}, \vartheta)
    + \pi_{\mathcal O}(s^{\prime}, o)\beta_{o}(s^{\prime}, \vartheta))}{\partial \vartheta} \right)^T \tilde{\varphi} = \frac{\partial u(o_0, s_1)}{\partial \vartheta}.
\end{align*}
Using Theorem 2, the natural gradient of the expected discounted return is:
\[\widetilde{\nabla}_{\vartheta} u(o_0, s_1) = G_{\vartheta}^{-1} \frac{\partial u(o_0, s_1)}{\partial \vartheta} = - \tilde{\varphi}.\]
\subsection*{Proof of Lemma 1}
\begin{lemma*}
Given intra-option policies, $\pi_{o}$ for all $o \in \mathcal O$, policy over options, $\pi_{\mathcal O}$, and terminations, $\beta_{o}$ for all $o \in \mathcal O$, then: 
\[\mathbb E[\delta^U_t | s_t, a_t, o_t, \pi_{o_t}, \pi_{\mathcal O}, \beta_{o_t}] = a_{U}(s_t, o_t, a_t).\]
\end{lemma*}
\begin{proof}
For time step $t$ we note that:
\begin{align}
    \mathbb E[\delta^U_t & | s_t, o_t, a_t, \pi_{o_t}, \pi_{\mathcal O}, \beta_{o_t}] \\
    =& \mathbb E [r_t + \gamma \hat{u}(o_t, S_{t+1}) - \hat{q}_{\pi_{\mathcal O}}(s_{t}, o_t) | s_t, o_t, a_t, \pi_{o_t}, \pi_{\mathcal O}, \beta_{o_t}] \\
    =& \sum_{s} \Pr(S_{t+1} {=} s| s_t, o_t, a_t) R(s_t, a_t, s) \\ 
    &+ \gamma \mathbb E[\hat{u}(o_t, S_{t+1}) | s_t, o_t,  a_t, \pi_{o_t}, \pi_{\mathcal O}, \beta_{o_t}] - q_{\pi_{\mathcal O}}(s_t, o_t) \label{eq:lemma1}.
\end{align}
Also,
\begin{align*}
    \mathbb E[&\hat{u}(o_t, S_{t+1}) | s_t, o_t, a_t, \pi_{o_t}, \pi_{\mathcal O}, \beta_{o_t}] \\
    =& \mathbb E \left [ \mathbb E[\hat{u}(o_t, S_{t+1}) | S_{t+1}, \pi_{\mathcal O}, o_t, \beta_{o_{t}}] | s_t, o_t, a_t, \pi_{o_t}, \pi_{\mathcal O}, \beta_{o_t} \right] \\
    =& \mathbb E[u(o_t, S_{t+1}) | s_t, o_t, a_t, \pi_{o_t}, \pi_{\mathcal O}, \beta_{o_t}] \\
    =& \sum_{s} \Pr(S_{t+1} {=} s |  s_t, o_t, a_t) u(s, o_t).
\end{align*}
Combining this expression with \eqref{eq:lemma1} we obtain:
\begin{align*}
    \mathbb E[\delta^U_t & | s_t, o_t, \pi_{o_t}, \pi_{\mathcal O}, \beta_{o_t}] \\ 
    =& \sum_{s} \Pr(S_{t+1} {=} s |  s_t, o_t, a_t) \Big (R(s_t, a, s) + \gamma u(o_t, s) \Big ) \\
    & -  q_{\mathcal O}(s_t, o_t) \\
    =& q_U(s_t, o_t, a_t) -  q_{\mathcal O}(s_t, o_t) = a_{U}(s_t, o_t, a_t).
\end{align*}
\end{proof}
\subsection*{Proof Of Lemma 2}
\begin{lemma*}
Under the precondition $o_t = o_{t - 1}$ and given intra-option policies, $\pi_{o}$ for all $o \in \mathcal O$, policy over options, $\pi_{\mathcal O}$, and terminations, $\beta_{o}$ for all $o \in \mathcal O$, then: \label{lemma:2_appendix}
\[\mathbb E[\delta^{\mathcal O}_t | s_{t}, o_t, o_{t} {=} o_{t-1}, \pi_{o_t}, \pi_{\mathcal O}] = a_{\mathcal O}(s_{t}, o_{t - 1}).\]
\end{lemma*}
\begin{proof}
For time step $t$ we note that:
\small
\begin{align}
    \mathbb E[\delta_t & |  s_{t}, o_t, o_{t} {=} o_{t-1}, \pi_{o_t}, \pi_{\mathcal O}] = \mathbb E[r_t + \gamma \hat{v}_{\pi_{\mathcal O}}(S_{t+1}) - \hat{v}_{\pi_{\mathcal O}}(s_t)]\\
    =& \sum_{s, o, a} \Pr(S_{t+1} {=} s, O_{t+1} {=} o, a_t = a | s_t, o_t) R(s_t, a, s) \\
    &+ \mathbb E[\hat{v}_{\pi_{\mathcal O}}(S_{t+1}) |  s_{t}, o_t, o_{t} {=} O_{t-1}, \pi_{o_t}, \pi_{\mathcal O}] - v_{\pi_{\mathcal O}}(s_t) \label{eq:lemma2}.
\end{align}
\normalsize
Also,
\small
\begin{align*}
    & \mathbb E[\hat{v}_{\pi_{\mathcal O}}(S_{t+1}) |  s_{t},  o, o_{t} {=} O_{t-1}, \pi_{p_t}, \pi_{\mathcal O}] \\
    =& \mathbb E[\mathbb E[\hat{v}_{\pi_{\mathcal O}}(S_{t+1}) | S_{t+1}, o_t, o_{t} {=} O_{t-1}, \pi_{o_t}, \pi_{\mathcal O}] \\
    &| s_{t}, o_t, o_{t} {=} O_{t-1}, \pi_{o_t}, \pi_{\mathcal O}] \\
    =& \mathbb E[v(S_{t+1}) | s_{t}, o_t, o_{t} {=} O_{t-1}, \pi_{o_t}, \pi_{\mathcal O}]\\
    =& \sum_{s, o, a} \Pr(S_{t+1} {=} s, O_{t+1} {=} o, a_t = a | s_t, o_t) q_{\pi_{\mathcal O}}(s, o).
\end{align*}
\normalsize
Combining this result with \eqref{eq:lemma2} we obtain:
\begin{align*}
    &\mathbb E[\delta_t |  s_{t}, o_t, o_{t} {=} O_{t-1}, \pi_{o_t}, \pi_{\mathcal O}] \\
    =& \sum_{s, o, a}  \Pr(S_{t+1} {=} s, O_{t+1} {=} o, a_t = a | s_t, o_t)\\
    & \left (R(s_t, a, s) + q_{\pi_{\mathcal O}}(s, o) \right) -  v_{_{\pi_{\mathcal O}}}(s_t)\\
    =& q_{_{\pi_{\mathcal O}}}(s_t, o_t) -  v_{_{\pi_{\mathcal O}}}(s_t) = q_{_{\pi_{\mathcal O}}}(s_t, o_{t - 1}) - v_{_{\pi_{\mathcal O}}}(s_t)\\
    =& a_{\mathcal O}(s_t, o_{t-1}).
\end{align*}
Notice that the penultimate equality follows from our assumption that $o_{t-1}$ is the same as $o_{t}$.
\end{proof}
\subsection*{Learning The Parameters $\eta$ And $\varphi$}
The partial derivative $\partial \epsilon(\eta, \theta)/\partial \eta$ upon observing state-action tuple $(s,a)$ when option $o$ is active:
\begin{align*}
    \frac{\partial \epsilon(\eta, \theta)}{\partial \eta} =& \frac{\ln \pi_{o}(s, a, \theta)}{\partial \theta} \frac{\ln \pi_{o}(s, a, \theta)}{\partial \theta}^T \eta\\
    &- a_U(s, o, a) \frac{\partial \ln \pi_{o}(s, a, \theta)}{\partial \theta}.
\end{align*}
The partial derivative $\partial \epsilon(\varphi, \vartheta)/ \partial \varphi$ upon entering the state $s'$ when option $o$ is active:
\begin{align*}
    \frac{\partial \epsilon(\varphi, \vartheta)}{\partial \varphi} =& - \frac{\partial \ln \beta_{o}(s', \vartheta)}{\partial \vartheta} \frac{\partial \ln \beta'_{o}(s', \vartheta)}{\partial \vartheta}^T \varphi \\
    &- \beta_{o}(s', \vartheta) a_{\mathcal O}(s', o) \frac{\partial \ln \beta_{o}(s')}{\partial \vartheta}.
\end{align*}
Note that the pair $s', o$ is shifted by one time step with $s'$ being one time step ahead of $o$. In the algorithm (INOC) we take steps in direction of the derivatives, provided above, to reduce the mean squared errors $\epsilon(\eta, \theta)$ and $\epsilon(\varphi, \vartheta)$. We use consistent estimates of the advantage functions $a_U$ and $a_{\mathcal O}$. We also maintain eligibility traces of the second term in the expressions above (\citealp{morimuranatgrad}; \citealp{Thomas2014BiasIN}).
\subsection*{Termination Gradient Updates}
\begin{figure*}
    \begin{subfigure}{.33\textwidth}
      \centering
      \includegraphics[width=.95\linewidth]{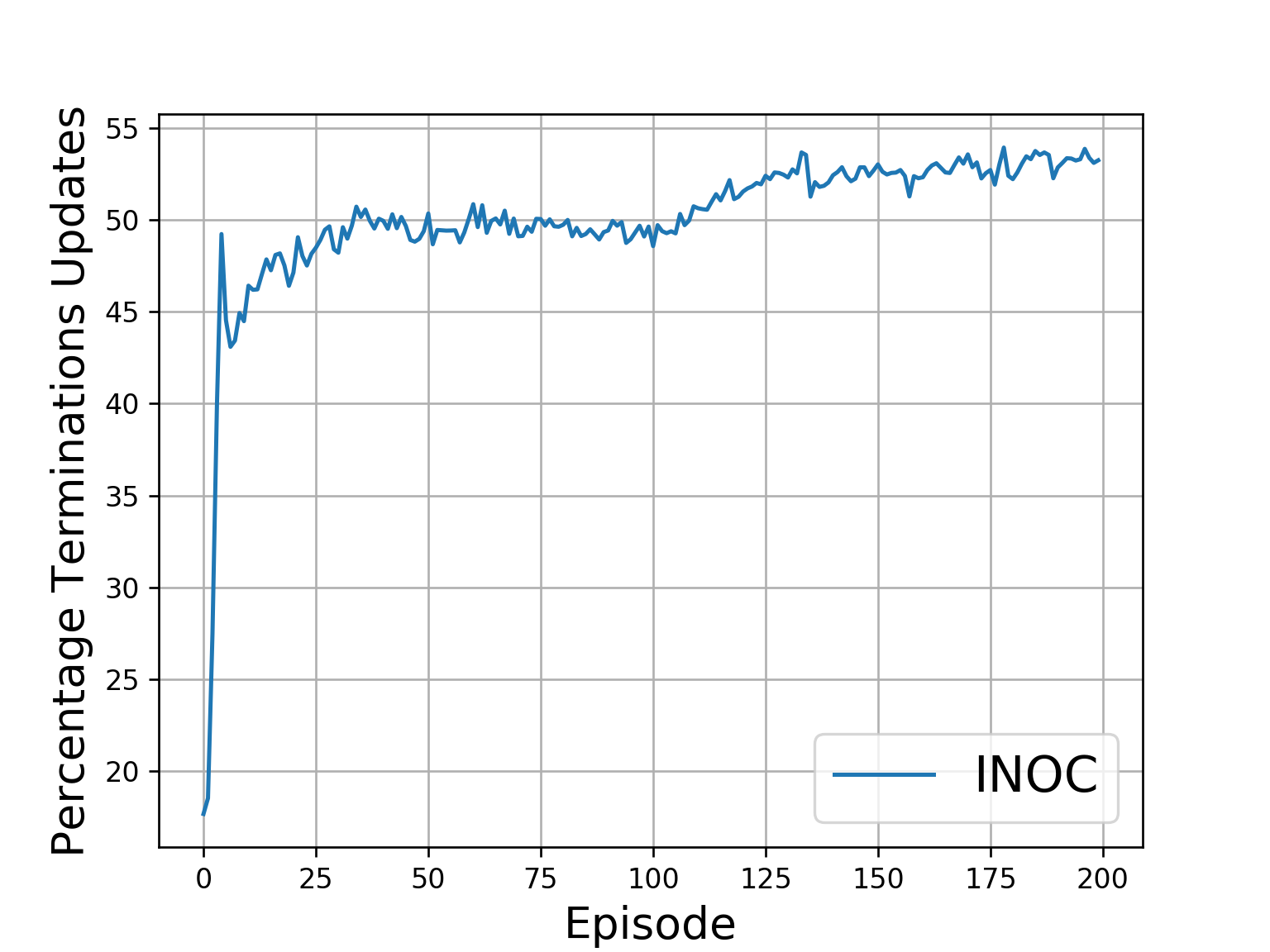}
      \caption{Asterisk}
      \label{fig:sfig1_aster}
    \end{subfigure}%
    \begin{subfigure}{.33\textwidth}
      \centering
      \includegraphics[width=.95\linewidth]{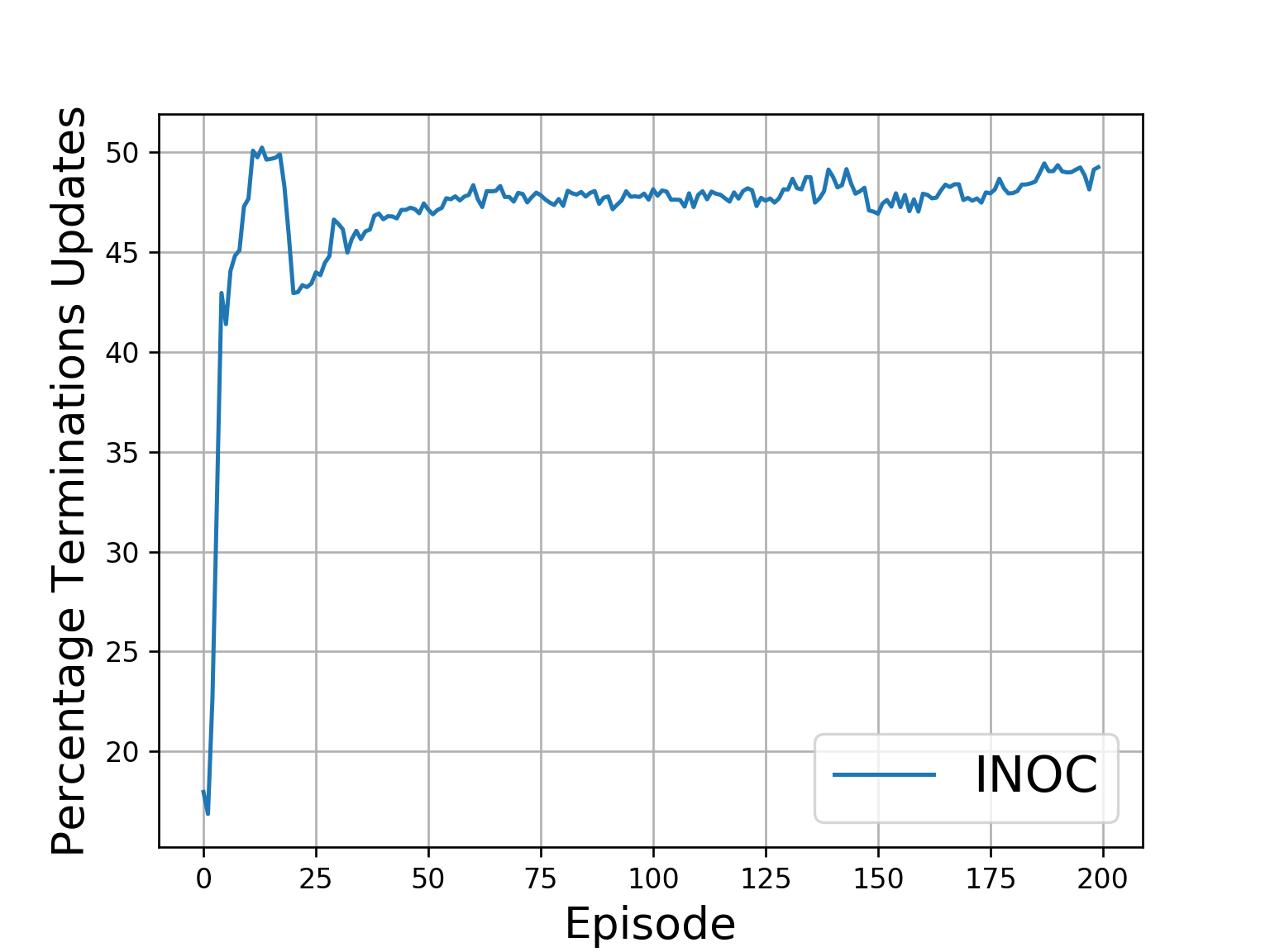}
      \caption{Seaquest}
      \label{fig:sfig2_sea}
    \end{subfigure}
    \begin{subfigure}{.33\textwidth}
      \centering
      \includegraphics[width=.95\linewidth]{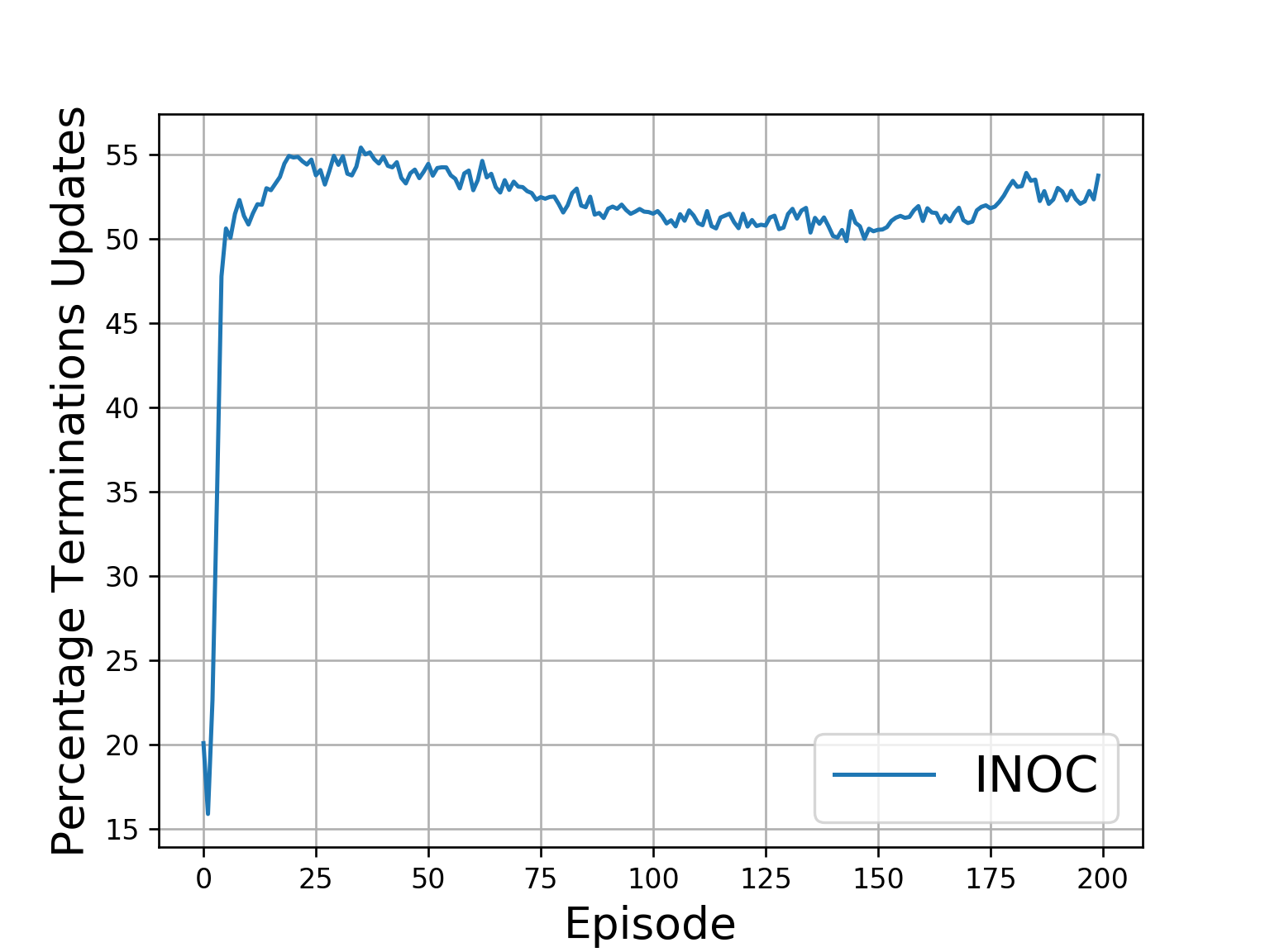}
      \caption{Zaxxon}
      \label{fig:sfig3_zax}
    \end{subfigure}
    \caption{The percentage of steps when an update is made to the termination policy for the Arcade Learning Environment, for a single trial}
    \label{fig:graphs_term_updates}
\end{figure*}
As noted in the section introducing the algorithm, the condition to get a consistent estimate of the advantage function, $a_{\mathcal O}$, which is that the previous option is the same as the current one, can limit the number of updates, this is evident from Figure \ref{fig:graphs_term_updates}. For Seaquest we note that the termination updates are sparser than the other three games. This suggests why the performance of option critic surpasses that of INOC for Seaquest.

\begin{figure*}
    \centering
    \includegraphics[width=.6\linewidth]{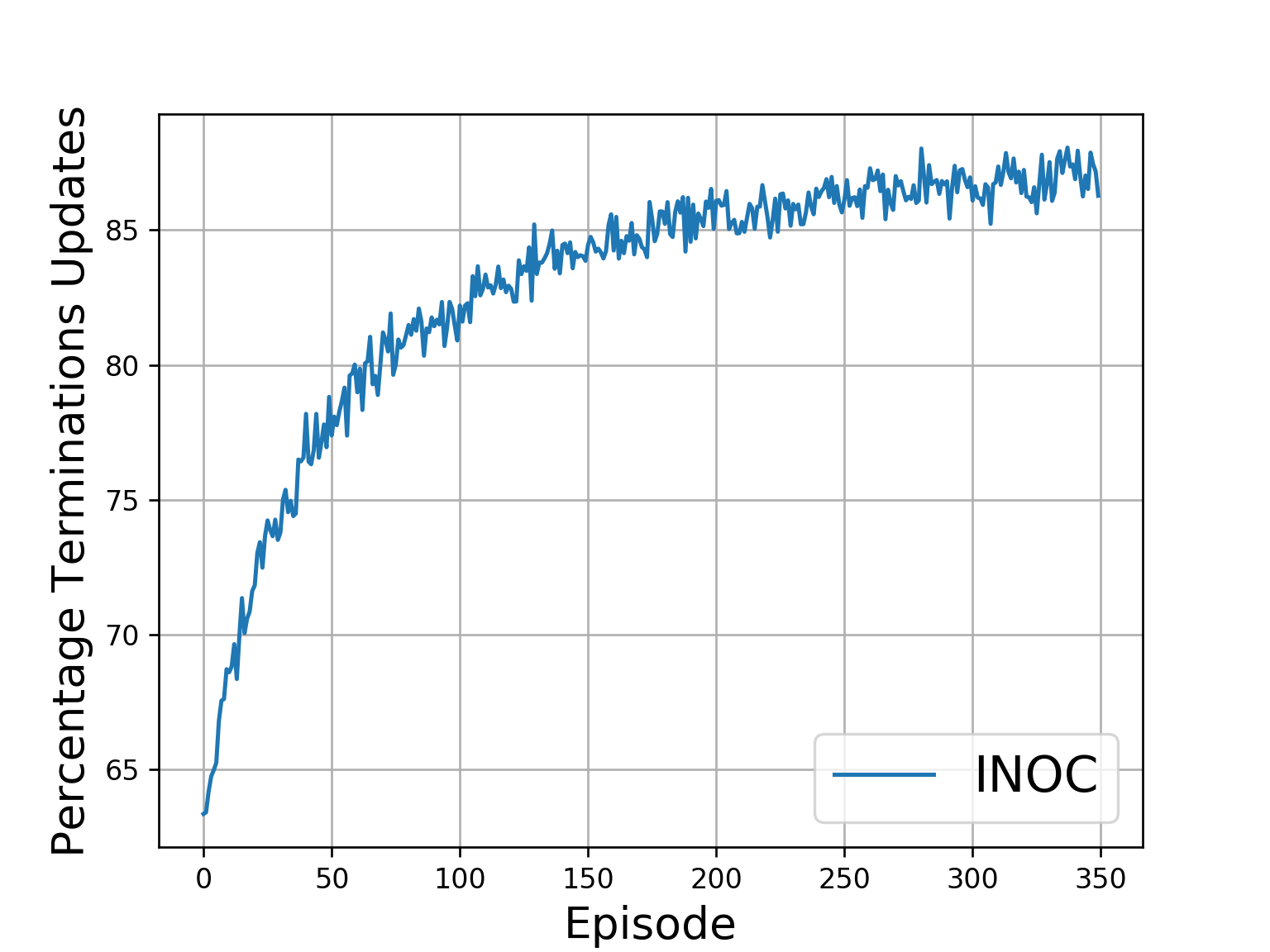}
    \caption{Percentage of steps where an update is made to the terminations in Four Rooms, averaged over 350 runs}
    \label{fig:four_rooms_term_update_ratio}
\end{figure*}

\end{document}